\documentclass{IEEEtran}
\usepackage{ista_unet_defs}
\usepackage{bm}
\usepackage{multirow}
\usepackage{booktabs}
\usepackage{siunitx}
\usepackage{graphicx}
\usepackage{amsfonts,amssymb}
\usepackage{mathtools}
\usepackage[dvipsnames,table,xcdraw]{xcolor}
\usepackage{xr-hyper}
\usepackage[colorlinks,filecolor =MidnightBlue, linkcolor=MidnightBlue,citecolor=MidnightBlue, pagebackref=true]{hyperref}
\usepackage{cite}
\usepackage{threeparttable}

\usepackage{glossaries}
\glsdisablehyper

\definecolor{grey}{RGB}{217,217,217}

\usepackage{etoolbox}
\makeatletter
\patchcmd{\@makecaption}
  {\scshape}
  {}
  {}
  {}
\makeatother
\def\tablename{Table}
\usepackage[format=plain]{caption}

\newacronym{cnn}{CNN}{convolutional neural network}
\newacronym{csc}{CSC}{convolutional sparse coding}
\newacronym{ct}{CT}{computed tomography}
\newacronym{mri}{MRI}{magnetic resonance imaging}
\newacronym{snr}{SNR}{signal-to-noise ratio}
\newacronym{psnr}{PSNR}{peak signal-to-noise ratio}
\newacronym{ssim}{SSIM}{structural similarity index Measure}
\newacronym{ista}{ISTA}{iterative shrinkage-thresholding algorithm}
\newacronym{sisr}{SISR}{single-image super-resolution}
\newacronym{siso}{SISO}{single-input single-output}
\newacronym{mimo}{MIMO}{multiple-input multiple-output}

\begin{document}
\title{Learning Multiscale Convolutional Dictionaries for Image Reconstruction}

\author{\IEEEauthorblockN{
Tianlin Liu\IEEEauthorrefmark{1},
Anadi Chaman\IEEEauthorrefmark{2}, 
David Belius\IEEEauthorrefmark{1} and
Ivan Dokmani\'c\IEEEauthorrefmark{1}\IEEEauthorrefmark{2}}

\IEEEauthorblockA{\IEEEauthorrefmark{1}Department of Mathematics and Computer Science, University of Basel.} 

\IEEEauthorblockA{\IEEEauthorrefmark{2}Department of Electrical and Computer Engineering, University of Illinois at Urbana-Champaign.}%
\thanks{This work was supported by the European Research Council Starting Grant 852821—SWING. The code to reproduce our experiments is available at \url{https://github.com/liutianlin0121/MUSC}}
}

\maketitle

\begin{abstract}
Convolutional neural networks (CNNs) have been tremendously successful in solving imaging inverse problems. To understand their success, an effective strategy is to construct simpler and mathematically more tractable convolutional sparse coding (CSC) models that share essential ingredients with CNNs. Existing CSC methods, however, underperform leading CNNs in challenging inverse problems. We hypothesize that the performance gap may be attributed in part to how they process images at different spatial scales: While many CNNs use multiscale feature representations, existing CSC models mostly rely on single-scale dictionaries. To close the performance gap, we thus propose a multiscale convolutional dictionary structure. The proposed dictionary structure is derived from the U-Net, arguably the most versatile and widely used CNN for image-to-image learning problems. We show that incorporating the proposed multiscale dictionary in an otherwise standard CSC framework yields performance competitive with state-of-the-art CNNs across a range of challenging inverse problems including CT and MRI reconstruction. Our work thus demonstrates the effectiveness and scalability of the multiscale CSC approach in solving challenging inverse problems.
\end{abstract}

\IEEEpeerreviewmaketitle

\section{Introduction}
\Glspl{cnn} obtain state-of-the-art performance in many image processing tasks. To understand their success, an active line of recent research reduces \glspl{cnn} into conceptually simpler and mathematically better-understood building blocks. Examples of these simplified convolutional models include \emph{convolutional kernels} \cite{Mairal2014convolutional, Mairal2016end, Bietti2019group}, \emph{convolutional scattering transforms} \cite{Mallat2012group, Bruna2013invariant, Zarka2020Deep, Zarka2021separation}, and \emph{convolutional sparse coding} \cite{Papyan2017convolutional,Papyan2018theoretical,Dror2019rethinking}. In addition to being mathematically tractable, these models have achieved remarkable empirical success, sometimes matching state-of-the-art \glspl{cnn}. 

This work studies convolutional representations arising from the \gls{csc} paradigm, which provides a natural connection between sparse representation models and \glspl{cnn}. Indeed, many \gls{cnn} instances can be interpreted as optimizing a \gls{csc} objective through cascaded layers \cite{Papyan2017convolutional}. Moreover, \gls{csc} models compete favorably with state-of-the-art \glspl{cnn} in several image processing tasks including denoising, single image super-resolution, and inpainting \cite{Degraux2017online,Chun2018convolutional, Dror2019rethinking, Zisselman2019local, ReyOtero2020variations, Chodosh2020when, Tolooshams2020convolutional, Arun2020cnn, Huang2020learning}.

While these emerging results are promising, the successful applications of \gls{csc} in imaging inverse problems are still confined to problems with relatively simple forward operators, including Gaussian noise addition, blurring, and uniformly random pixel removal. Common to these forward operators is their \emph{spatial locality} -- they introduce artifacts that are spatially correlated only, if at all, within small pixel neighbourhoods. By contrast, a broad range of imaging inverse problems involve forward models that mix distant parts of the image and are \emph{highly spatially heterogeneous}; examples include the Radon transform for computed tomography, which computes line integrals along radiating paths, and the Fourier transform for magnetic resonance imaging, which computes inner products with globally-supported sinusoids. Working with these forward models presents different challenges since they introduce structured noise, such as streak artifacts, with long-range spatial correlations. We thus ask a natural question: Can \gls{csc} models also yield strong performance on such inverse problems with non-local operators?

To deal with spatially heterogeneous imagery data, one natural strategy is to employ \emph{multiscale} dictionaries. Indeed, seminal works have shown that multiscale dictionaries, either analytical or learned, are advantageous in representing and processing images \cite{Mallat1989theory, Candes2004newtight, Starck2002curvelet, Mairal2008learning, Boaz2011multi, Maggioni2016multiscale}. Separating scales is useful because it gives efficient descriptions of structural correlations at different distances. Yet, these existing \gls{csc} models \cite{Dror2019rethinking, Lecouat2020fully, Lecouat2020flexible, Khatib2021learned, Scetbon2021deep} mostly employ single-scale dictionaries, whose dictionary atoms all have the same size. While there exist proposals for multiscale \gls{csc} architectures, they are tailored for specific tasks \cite{Li2018video, Chang2018unsupervised}. In addition, \gls{csc} models do away with flexible skip connections between non-consecutive layers, which are nonetheless essential for many successful \glspl{cnn} such as the U-Net and its variants \cite{Ronneberger2015unet, Falk2019unet, Zhou2020unetplusplus} to fuse features across scales. This challenge of harnessing multiscale features in the \gls{csc} paradigm motivates our work.

To address the challenge, we introduce a multiscale convolutional dictionary inspired by the highly successful U-Net \cite{Ronneberger2015unet}. We then apply the multiscale convolutional dictionary to challenging, large-scale inverse problems in imaging. The main contribution of this paper is twofold:

\begin{itemize}
    \item We propose a new convolutional dictionary, whose representation incorporates atoms of different spatial scales. The proposed multiscale dictionary augments standard, single-scale convolutional dictionaries to exploit the spatially-heterogeneous properties of images.  
    
    \item We study the effectiveness of the multiscale convolutional dictionary through experiments on large-scale datasets. We find that the performance of the multiscale \gls{csc} approach is competitive with leading \glspl{cnn} on datasets including two major CT and MRI benchmarks. We additionally show that our model matches (and slightly improves) the state-of-the-art performance on the deraining task achieved by a deep neural network \cite{Ren2019progressive}. Notably, the single-scale CSC model performs significantly worse on this task \cite{Khatib2021learned}.
\end{itemize}

Overall, our work makes a step forward in closing the performance gap between end-to-end \glspl{cnn} and sparsity-driven dictionary models. At a meta level, it (re)validates the fundamental role of sparsity in representations of images and imaging operators \cite{beylkin1991fast,bruckstein2009sparse,Candes2004newtight}.

The rest of this article is organized as follows. In Section \ref{sec:background}, we first briefly review the sparse representation model and its relationship to \glspl{cnn}. Section \ref{sec:method} explains how we incorporate multiscale atoms in a dictionary model; we also explain how to learn the multiscale dictionary from data under the task-driven dictionary learning framework. Section \ref{sec:experiments} reports experimental results on tasks including CT reconstruction and MRI reconstruction. 

\section{Background and Related Work \label{sec:background}}
In this section, we briefly review the related work; a summary of notation is given in Table~\ref{table_natations}.

\subsection{Sparse representation models} Sparse representation has been extensively studied and widely used in imaging inverse problems \cite{lustig2008compressed,charlety2013global, Liebgott2013Prebeamformed}. It is motivated by the idea that many signals, images being a prime example, can be approximated by a linear combination of a few elements from a suitable overcomplete basis. The sparse representation framework posits that we can decompose a signal of interest\footnote{For simplicity we write all signals as 1d vectors, but the formulation is valid in any dimension.} $\vz \in \RR^d$ as $\vz = \mD \valpha$, where $\mD \in \RR^{d \times N}$ is an overcomplete dictionary of $N$ atoms ($N > d$) and $\valpha \in \RR^N$ is a sparse vector with few non-zero entries. Learning a sparse representation model thus comprises two sub-problems: (i) given a dictionary $\mD$, encode the signal $\vz$ into a sparse vector $\valpha$ (the \emph{sparse coding} problem), and (ii) given a set of signals, learn an appropriate dictionary $\mD$ that sparsifies them (the \emph{dictionary learning} problem). We briefly review these two problems and show how they are related to neural network models such as \glspl{cnn}.

\subsection{The sparse coding problem} The sparse coding problem is often formulated as basis pursuit denoising \cite{Chen1994basis} or Lasso regression \cite{Tibshirani1996regression}. Most relevant to our work is its formulation with non-negative constraints on the sparse code $\valpha$:
\begin{equation}
  \label{eq:lasso}
  \minimize_{\valpha \geq 0}~~ \frac{1}{2}\norm{\vz - \mD \valpha}_2^2 + \lambda \norm{\valpha}_1.
\end{equation}
Here, the first term $\frac{1}{2}\norm{\vz - \mD \valpha}_2^2$ ensures that the code $\valpha$ yields a faithful representation of $\vz$, the second term $\lambda \norm{\valpha}_1$ controls the sparsity of the code, and the two terms are balanced by a parameter $\lambda > 0$. 
An effective solver for the minimization problem \eqref{eq:lasso} is the \gls{ista} \cite{Daubechies2004iterative}, which executes the following iteration
\begin{equation}
  \label{eq:ista-execution}
\begin{split}
  \valpha^{[k+1]} & \coloneqq \Scal(\valpha^{[k]}, \vz; \mD, \vlambda) \\
  & \coloneqq \sigma (\valpha^{[k]} + \eta \mD^\top(\vz - \mD \valpha^{[k]}) - \eta \vlambda),
\end{split}
\end{equation}
where the superscript $[k]$ denotes the iteration number, $\eta$ is a step-size parameter, $\vlambda$ is a vector whose entries are all $\lambda$, and $\sigma(x) \coloneqq \max (x,0)$ is a component-wise rectifier function. For simplicity, we use $\Scal(\valpha, \vz; \mD, \vlambda)$ to denote one execution of ISTA with measurement $\vz$, sparse code $\valpha$, dictionary $\mD$, and threshold $\vlambda$. The ISTA algorithm is a composition of such executions; we write $\mathrm{ISTA}_K$ for the $K$-fold composition of $\Scal$ with itself:
\begin{equation} \label{eq:ista-k-steps}
\begin{split}
& \ista_{K} (\vz; \mD, \vlambda ) \\
& \coloneqq \underbrace{\Big ( \Scal ( \cdot, \vz; \mD, \vlambda) \circ\cdots\circ \Scal ( \cdot, \vz; \mD, \vlambda) \Big )}_{K \text{ times}} (\valpha^{[0]}),
\end{split}
\end{equation}
where $\valpha^{[0]}$ is the initial sparse code; throughout this work, this initial code $\valpha^{[0]}$ is assumed to contain zero in all entries. We emphasize that ISTA is a nonlinear transform of its input $\vz$.

\begin{table}[!tbp] \normalsize
	\caption{The notations used in this paper.}
	\centering
		\begin{tabular}{|ll|} \hline \rule{0mm}{4mm}
		    $d$ & the dimension of an image \\
		    $N$ & the dimension of a sparse code \\
		    $M$ & the number of training samples \\
		    $K$ & the number of sparse coding steps \\
		    $\sigma(\cdot)$ & the ReLU non-linearlity \\
		    $\vz$ & a noisy image to be processed \\
			$\mD$ & an overcomplete dictionary \\
		    $\vgt$ & a ground-truth image \\
		    $\vapprox$ & a predicted image \\
			$\valpha$ & a sparse code \\
		    $\vlambda$ & the thresholding parameters of ISTA \\[1mm] 
			\hline
		\end{tabular}
	\label{table_natations}
\end{table}

\subsection{The task-driven dictionary learning problem}
We now briefly recall the task-driven dictionary learning framework \cite{Mairal2012task}. Consider a supervised learning setting, in which we aim to identify a parametric function that associates each input $\vz$ (e.g., a corrupted image) with its target $\vgt$ (e.g., a clean image) for all $(\vz, \vgt) \in \RR^d \times \RR^d$ drawn from some joint distribution. In the task-driven framework, we proceed by first representing the signal $\vz$ by a sparse code $\valpha_{\vz}$ with respect to a dictionary  $\mD$. One way to achieve this is to let
\begin{equation} \label{eq:sparse_coding_objective}
\valpha_{\vz} \coloneqq \argmin_{\valpha \geq 0} \frac{1}{2}\norm{\vz - \mD \valpha}_2^2 + \lambda \norm{\valpha}_1,    
\end{equation}
which can be approximated by $K$ iterations of \gls{ista} as in \eqref{eq:ista-k-steps}. Next, we approximate the desired target $\vgt$ using the sparse code $\valpha_{\vz}$ through a regression model $f( \cdot , \vw)$ with learnable parameter $\vw$. For instance, $f( \cdot , \vw)$ could be a linear regression model with weights and biases $\vw$. The model output $f( \valpha_{\vz} , \vw)$ thus depends on the regression model parameters $\vw$ as well as the sparse code $\valpha_{\vz}$, which in turn depends on the dictionary $\mD$ through the ISTA iterations. In this way, the regression parameters $\vw$ and dictionary $\mD$ can be \emph{jointly} optimized, for instance, with respect to the quadratic loss objective evaluated on a dataset of $M$ input-target pairs $\{(\vz_i, \vgt_i)\}_{i = 1}^M$:
\begin{equation} \label{eq:task_driven_objective}
\minimize_{\vw, \mD}~~ \frac{1}{2M} \sum_{i = 1}^M  \| f( \valpha_{\vz_i} , \vw) -  \vgt_i \|_2^2 .    
\end{equation}
Importantly, the task-driven objective in \eqref{eq:task_driven_objective} implies that the dictionary $\mD$ is optimized to solve the supervised learning task and not just to sparsely represent data.
\begin{figure*}[t]
\centering
\includegraphics[width= \textwidth]{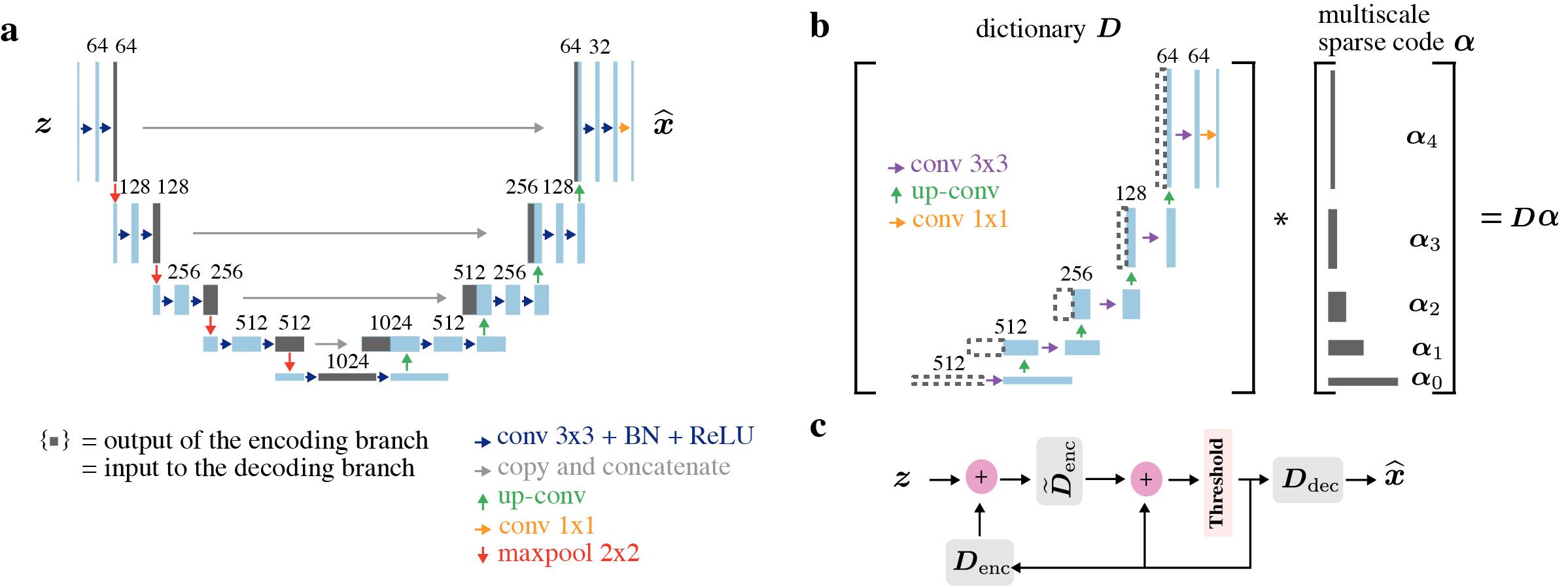}
\caption{\textbf{Schematic illustration of the U-Net (left panel) and the dictionary model considered in this work (right panel).} (\textbf{a}): The U-Net processes input images using convolution, scale separations, and skip connections in conjunction with ReLU non-linearities and batch-normalization (BN) modules indicated by colored arrows. The multi-channel feature maps produced by these operations are illustrated as boxes with the channel numbers indicated on the top of these boxes. Dark grey boxes indicate the feature maps produced by the encoding branch of the U-Net, which are sent to the decoding branch either through channel-wise concatenation (``skip connection'') or through the bottleneck layer. (\textbf{b}): The dictionary considered in this work is a simplification of the decoder branch of the U-Net: We retain convolution and multiscale representation from the decoder branch of the U-Net but remove all non-linearities, batch-normalization, and additive biases; additionally, we remove a convolution at each spatial resolution level and halve the number of convolutional channels for all convolutions. Grey boxes indicate the multiscale sparse code $\valpha = (\valpha_0, \ldots, \valpha_4)$ that the dictionary takes as input. Dashed boxes indicate the position that each $\valpha_i$ feed into the dictionary. (\textbf{c}): The proposed  as a computational graph that uses multiscale dictionaries $\encdict$, $\adjencdict$, and $\decdict$; although each dictionary is linear, the computational graph is nonlinear due to the thresholding operator.
 \label{fig:unet_schematic}}   
\end{figure*}

\subsection{Convolutional sparse coding \label{sec:csc}} Our work is inspired by the convolutional sparse coding (CSC) model \cite{Bristow2013fast,Sorel2016fast,Liu2017online, Sreter2018learned,Papyan2017convolutional}, which bridges deep \glspl{cnn} and sparse representation models. Concretely, Papyan et al. \cite{Papyan2017convolutional} noticed that if the dictionary $\mD$ has a convolutional structure and if the sparse code $\valpha$ is assumed to be non-negative, a single iteration of \gls{ista} with~$\valpha^{[0]}$ initialized as a zero vector and step-size $\eta = 1$ is equivalent to the forward pass of a single-layer convolutional network
\begin{equation}
\valpha = \sigma(\mD^\top\vz + \vb),
\end{equation}
where $\vb$ is a vector whose components are $-\lambda$ (cf. Equation~\eqref{eq:ista-execution}). This single-layer formulation can be extended to characterize a deep \gls{cnn} of multiple layers. Specifically, the forward-pass of a deep \gls{cnn} of $L$-layers can be interpreted to approximate the sparse codes $\valpha_1, \cdots, \valpha_L$ sequentially with respect to different dictionaries $\mD_1, \cdots, \mD_L$; the back-propagation pass is interpreted as an update to these dictionaries~$\{\mD_i\}_{i=1}^L$ in a task-driven way.

\subsection{CNNs for solving inverse problems}

Deep \glspl{cnn} achieve state-of-the-art performance in many image processing tasks \cite{Baguer2020computed, Leuschner2021quantitative, Huang2021deep, Dong2019denoising}. In particular, the U-Net \cite{Ronneberger2015unet} and its variants \cite{Falk2019unet, Zhou2020unetplusplus, Etmann2020iunets} are among the most extensively used \gls{cnn} architectures in solving image-to-image learning tasks. U-Nets represent images via multiscale features computed from measurements using an encoding (or downsampling) branch and a synthesized into an estimated image in a decoding (or upsampling) branch (Figure~\ref{fig:unet_schematic}\textbf{a}). In the downsampling branch, the spatial resolutions of feature maps are reduced while the number of feature maps is increased; in the decoding branch, these features are recombined with previous high-resolution features via channel concatenation (``skip connections'') and convolution. Heuristically, low-resolution feature maps of a U-Net capture large-scale image properties, whereas the high-resolution feature maps capture more fine-grained image properties \cite{Etmann2020iunets}. In a related line of work, Ye et al. \cite{Ye2018deep, Han2018framing, Ye2019understanding} proposed to use the framelets formalism \cite{Yin2017tale} to study aspects of U-Net-like encoder-decoder \glspl{cnn}. A key observation they make is that a U-Net model is closely related to convolutional framelets whose frame basis selection depends non-linearly on input data.

\section{CSC with multiscale dictionaries \label{sec:method}}
The structure of a convolutional dictionary is crucial to a \gls{csc} model since the dictionary atoms characterize the signals that can be represented sparsely. In the existing formulation of CSC, atoms of a convolutional dictionary have a single scale, in the sense that they all share the same spatial shape. However, many image classes and imaging artifacts exhibit structured correlations over multiple scales. To exploit these correlations in imaging inverse problems, we construct multiscale convolutional dictionaries.

Our construction is based on the U-Net reviewed in Section \ref{sec:background}. Indeed, the tremendous success of U-Nets has in part been attributed to their ability to represent images at multiple scales \cite{Zhou2020unetplusplus, Ye2019understanding}, which is achieved by using up- and downsampling operations together with skip connections as in Figure~\ref{fig:unet_schematic}\textbf{a}. Another property of the U-Net is its shared parameters across scales: Low-resolution features (the grey boxes at the bottom of Figure~\ref{fig:unet_schematic}\textbf{a}) and high-resolution features (the grey boxes at the top of Figure~\ref{fig:unet_schematic}\textbf{a}) undergo an overlapping synthesizing path parameterized by shared weights. This weight-sharing strategy has not been employed by existing proposals for multiscale CSC dictionaries \cite{Li2018video, Chang2018unsupervised}. In what follows, we describe the construction process of a \textit{linear} dictionary inspired by and closely following the standard U-Net. 

\subsection{Encoder--decoder dictionaries} We denote the encoding branch of the U-Net by $f_{\mathrm{enc}}( \cdot, \vthetaenc): \RR^d \to \RR^N$ with parameters $\vthetaenc$; the encoding branch maps the input $\vz \in \RR^d$ to convolutional feature maps $\valpha_{\vz} = f_{\mathrm{enc}}( \vz,  \vthetaenc) \in \RR^N$, illustrated as the dark grey boxes in Figure~\ref{fig:unet_schematic}\textbf{a}. Note that, for a U-Net, the intermediate feature map dimension $N$ (number of scalar coefficients in $\valpha$) is typically much greater than the image dimension $d$. These feature maps are then fed into the decoding branch of the U-Net either through skip connections or through the bottleneck layer. To describe this process, we write the decoding branch of the U-Net as a function $f_{\mathrm{dec}}( \cdot, \vthetadec): \RR^N \to \RR^d$ with parameters $\vthetadec$. That is, the function $f_{\mathrm{dec}}( \cdot, \vthetadec)$ takes the convolutional feature maps produced by the encoding branch and transforms them to produce the model output. We can thus write the output produced by a U-Net as
\[ \widehat{\vz} \coloneqq f_{\mathrm{dec}} (\valpha_{\vz}, \vthetadec ) = f_{\mathrm{dec}} \Big (f_{\mathrm{enc}}(\vz,  \vthetaenc), \vthetadec \Big ).\]

We now focus on the image synthesis process of the U-Net, described by the decoding function $f_{\mathrm{dec}}( \cdot, \vthetadec)$. This function synthesizes convolutional feature maps at different spatial scales through skip connections and upsampling. As such, the decoding branch of the U-Net approximates an image $\vgt \in \RR^d$ using multiscale feature maps $\valpha_{\vz} \in \RR^N$ of a much higher dimension, so that $\vgt \approx f_{\mathrm{dec}} (\valpha_{\vz}, \vthetadec )$. Conceptually, this representation is similar to the sparse and overcomplete representation in a dictionary, except that the U-Net decoder is non-linear.

To construct a multiscale dictionary, we thus consider a stripped-down version of the image synthesis process of U-Net by removing all non-linearities, batch normalization, and additive biases from the function $f_{\mathrm{dec}} (\cdot, \vthetadec)$, as shown in Figure~\ref{fig:unet_schematic}\textbf{b}; to further simplify the architecture, at each spatial scale, we additionally remove a convolution and halve the number of convolutional channels for all convolutions.  The resulting function is then simply a \emph{linear} transformation 
\begin{equation} \label{eq:dictionary-synthesis}
\valpha \coloneqq (\valpha_0, \ldots, \valpha_4) \mapsto \decdict \valpha,
\end{equation}
where $\valpha_0, \ldots, \valpha_4$ are sparse code having different resolutions (visualized as the grey boxes in Figure 1\textbf{b}). This dictionary shares the essential ingredients of convolution, multiscale representation, and skip connections with the U-Net decoding branch and therefore we refer to it as the decoder dictionary. A precise definition of the decoder dictionary $\decdict$ through convolution and upsampling is provided in Appendix \ref{appendix:decoder_dictionary_description}.

\subsection{The dictionary-based sparsity prior}

With a given decoder dictionary $\decdict$ to describe the image synthesis process, we next consider how to infer an associated sparse code $\valpha$, so that $\decdict \valpha$ is a good approximation of the image we wish to model. In a supervised learning setting where the input image $\vz$ is given, it is natural to interpret $\valpha$ as an encoded representation of $\vz$. Since the encoding must produce a coefficient vector whose structure is compatible with $\valpha$, we endow an encoder dictionary $\encdict \in \RR^{d \times N}$ with the same structure of $\decdict$ albeit with a different set of atoms. This setup is analogous to U-Net's encoding and decoding branches: the encoder and decoder dictionaries $\encdict$ and $\decdict$ are employed to process input signals and produce output signals, respectively. The sparse code $\valpha_{\vz}$ induced by an input $\vz$ and the encoder dictionary $\encdict$ then facilitate the subsequent task for approximating the ground-truth image $\vx$:
\begin{equation} \label{eq:enc_dec_dictionaries}
     \vz \xrightarrow{\hspace{1mm} \textrm{ Sparse coding with } \encdict \hspace{1mm}} \valpha_{\vz}  \xrightarrow{\hspace{1mm} \textrm{ Synthesis with } \decdict \hspace{1mm}} \vapprox \coloneqq \decdict \valpha_{\vz}.
\end{equation}

In what follows, we derive a supervised learning method that turns each $\vz$ into a prediction $\vapprox$ using encoder and decoder dictionaries.

\subsection{The task-driven dictionary learning objective \label{subsec:the-task-driven-objective}} Under the task-driven framework introduced in Section \ref{sec:background}, we formulate a supervised learning problem via sparse coding and dictionary learning. We consider the following minimization problem over a dataset of $M$ input-target pairs $\{(\vz_i, \vgt_i) \}_{i = 1}^M$:
\begin{equation} \label{eq:ista_unet_objective}
\begin{aligned}
\minimize_{\{\encdict,\decdict\}, \vlambda > 0  } \quad & \frac{1}{2M} \sum_{i = 1}^M \| \decdict \valpha_{\vz_i} - \vgt_i  \|_2^2  \\
\text{where} \quad & \valpha_{\vz_i} \coloneqq \ista_{K} (\vz_i; \encdict, \vlambda ). 
\end{aligned}
\end{equation}
The objective in \eqref{eq:ista_unet_objective} penalizes the discrepancy between the ground-truth signal $\vgt$ and the model prediction $\decdict \valpha_{\vz}$, where the latter is a signal synthesized from a sparse code $\valpha_{\vz}$ via the decoder dictionary $\decdict$; the code $\valpha_{\vz}$ is a sparse representation of the input image $\vz$ with respect to the encoder dictionary $\encdict$ by unrolling a fixed number $K$ of \gls{ista} iterations. The sparsity-controlling parameter $\vlambda$ is multi-dimensional, weighting codes component-wise. The intuition behind this choice is that the different convolutional features, especially those at different scales, should be thresholded differently. The sparse code $\valpha$, illustrated as the grey boxes in Figure~\ref{fig:unet_schematic}\textbf{b}, is a collection of multi-dimensional tensors, each corresponds a spatial scale. 

The task driven objective \eqref{eq:ista_unet_objective} defines a computational graph that transforms an input image $\vz$ into a prediction $\decdict \valpha_{\vz}$. We term this computational graph MUSC, since it involves \textbf{m}ultiscale \textbf{U}-Net-like \textbf{s}parse \textbf{c}oding. We note the MUSC is an instance of \emph{optimization-driven networks} \cite{Lecouat2020flexible} derived by unrolling an optimization algorithm. It incorporates two modules with meaningful objectives, one implementing sparse coding and the other dictionary-based synthesis. This composition is arguably conceptually more interpretable than end-to-end layerwise composition of deep networks. 

While a traditional compressed sensing approach uses \emph{a single dictionary} for reconstruction, our approach uses \emph{two dictionaries} $\encdict$ and $\decdict$ in the task-driven learning objective \eqref{eq:ista_unet_objective}. This discrepancy is due to different assumptions in measurement-to-image reconstruction (the compressed sensing approach) and image-to-image reconstruction (our approach).  Consider an inverse problem with a forward operator $\mA$, a unknown ground-truth signal $\vgt$, and measurements $\vy \coloneqq \mA \vgt$; in CT reconstruction, $\mA$ is the Radon transform and $\vy$ is the measured sinogram. The compressed sensing approach estimates $\vgt$ as $\mD \valpha^\ast$ for some dictionary $\mD$, where 
\begin{equation} \label{eq:classical-cs}
\valpha^\ast = \argmin_{\valpha} \norm{\mA \mD \valpha - \vy}_2 + \lambda \norm{\valpha}_1
\end{equation}
is the inferred sparse code based on the dictionary $\mD$. Note that   \eqref{eq:classical-cs} and the synthesis $\mD \valpha^*$ require only a single dictionary $\mD$. However, this approach assumes that we know the measurements $\vy$ and the forward operator $\mA$.

If we were to apply a single dictionary $\mD \coloneqq \encdict = \decdict$ in our image-to-image learning approach in \eqref{eq:ista_unet_objective}, we would find a sparse code $\valpha$ such that $\mD \valpha \approx {\vgt}$ and $\mD \valpha \approx {\mA^+ \mA \vgt}$. This is difficult when $\mA^+ \mA$ significantly differs from the identity operator as in the case of highly ill-posed problems. On the other hand, using two dictionaries $\encdict$ and $\decdict$ in \eqref{eq:ista_unet_objective} requires finding a sparse code $\valpha$ such that $\decdict \valpha \approx {\vgt}$ and $\encdict \valpha \approx {\mA^+ \mA \vgt}$, a formulation that is more flexible when $\mA^+ \mA$ substantially differs from the identity. Experiments in  Section~\ref{subsec:ablation-study} confirm that allowing $\encdict \neq \decdict$ yields better performance. We note that our approach is morally related to setting $\encdict = \mA \mD$ in \eqref{eq:classical-cs}, but since we do not know $\mA$ we have to learn $\encdict$ from samples together with $\decdict$. Such a learned encoder dictionary captures information about $\mA$, entangled with information about the data distribution.

\subsection{Relaxation on dictionaries \label{subsec:relaxation_regularization}}

We now describe computational techniques that stabilize the gradient-descent-based dictionary learning of MUSC. Following earlier work \cite{Gregor2010learning, Dror2019rethinking, Lecouat2020flexible, Lecouat2020fully,Zarka2020Deep}, we untie the encoder dictionary from its adjoint during dictionary update. That is, we replace the execution in \eqref{eq:ista-execution} by
\begin{equation} \label{eq:relax-ista-execution}
\widetilde{\Scal}(\valpha, \vz; \encdict, \adjencdict, \vlambda) \coloneqq \sigma (\valpha + \eta \adjencdict^\top(\vz - \encdict \valpha) - \eta \vlambda),
\end{equation}
where the dictionary $\adjencdict$ is initialized to be identical to $\encdict$ but is allowed to evolve independently during training. Even though the theoretical effects of this relaxation remain unclear, the dictionary $\adjencdict$ can be interpreted as a learned preconditioner that accelerates training \cite{Lecouat2020flexible, Lecouat2020fully}; see also the investigation in \cite{Moreau2017understanding, Liu2018alista, Zarka2020Deep}. The \emph{learned ISTA} (LISTA) algorithm \cite{Gregor2010learning} corresponding to \eqref{eq:relax-ista-execution} is written as
\begin{equation} \label{eq:relax-ista}
\begin{split}
 & \lista_{K} (\vz; \encdict, \adjencdict, \vlambda ) \\
 & \small \coloneqq \underbrace{\Big ( \widetilde{\Scal} ( \cdot, \vz; \encdict, \adjencdict, \vlambda_K) \circ\cdots\circ \widetilde{\Scal} ( \cdot, \vz; \encdict, \adjencdict, \vlambda_1) \Big )}_{K \text{ times}} (\valpha^{[0]}),
\end{split}
\end{equation}
where $\vlambda_1, \ldots, \vlambda_K$ are the soft-thresholding parameters for each ISTA execution. Note that, in \eqref{eq:relax-ista}, the soft-thresholding parameters $\{\vlambda_i\}_{i=1}^K$ depend on the execution step. As shown in \cite{Zarka2020Deep}, incorporating step-dependent soft-thresholding parameters can be beneficial. While \cite{Zarka2020Deep} uses a homotopy continuation strategy to adjust these parameters we treat them as learnable parameters for simplicity. Taking these considerations into account, we define a new regression loss:
\begin{equation} \label{eq:ista_unet_relax_objective}
\begin{aligned}
& \mathcal{L}(\encdict, \decdict, \vlambda) \coloneqq \frac{1}{2M} \sum_{i = 1}^M \| \decdict \valpha_{\vz_i} - \vgt_i  \|_2^2, \\
& \text{where} \quad \valpha_{\vz_i} =  \lista_{K} (\vz_i; \encdict, \adjencdict, \vlambda ). \\
\end{aligned}
\end{equation}
Unless mentioned otherwise, we use the loss \eqref{eq:ista_unet_relax_objective} to train MUSC throughout our paper. In Section \ref{subsec:ablation-study}, we compare the performance of trained model using \eqref{eq:ista_unet_relax_objective} and \eqref{eq:ista_unet_objective}.

\subsection{Training the MUSC}

Training the MUSC entails the following three steps:
  
\begin{enumerate}
    \item \textbf{Dictionary initialization:} We randomly initialize the dictionary $\encdict$ and initialize $\decdict$, and $\adjencdict$ as identical copies of $\encdict$.
    \item \textbf{Model forward pass:} For each input image $\vz_i$, we evaluate the model prediction $\decdict \valpha_{\vz_i}$ as in Equation \eqref{eq:ista_unet_relax_objective}. For ISTA executions, we initialize all sparse code $\valpha_{\vz}^{[0]}$ as a collection of all-zero tensors; the ISTA step-size parameter $\eta$ is initialized as the inverse of the dominant eigenvalue of the matrix $\encdict^\top \encdict$, which can be approximated using by power iteration (Appendix \ref{appendix:power_iter}).
    \item \textbf{Task-driven dictionary learning:} For a mini-batch of input-target pairs, solve the optimization problem in \eqref{eq:ista_unet_relax_objective} with gradient descent.
\end{enumerate}

\begin{figure*}[ht]
\centering
\includegraphics[width= \textwidth]{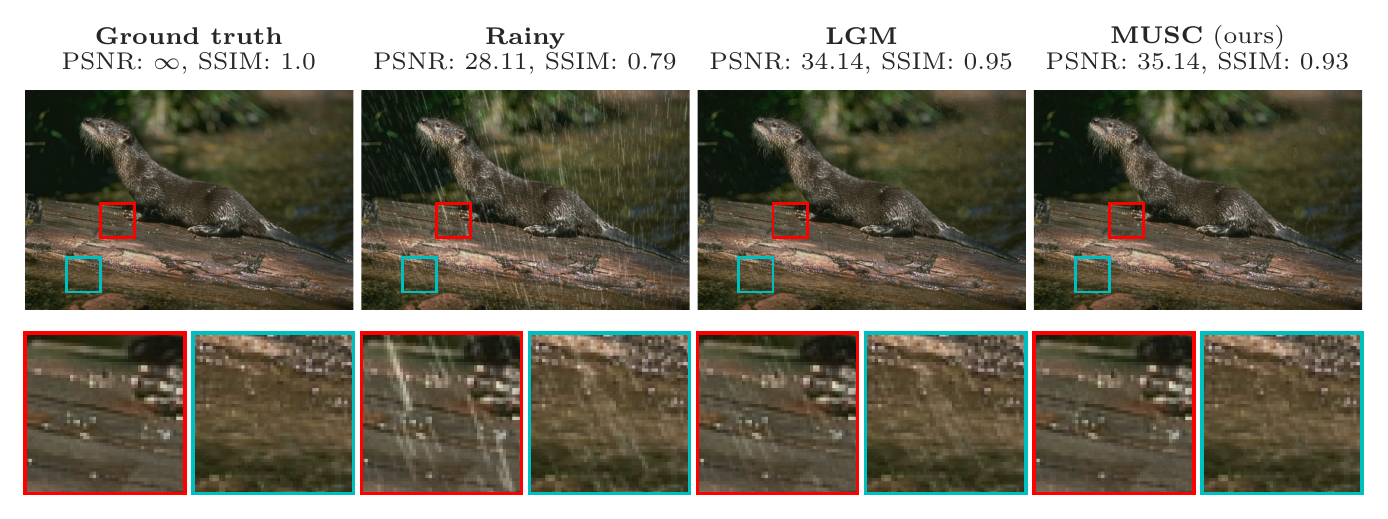} 
\caption{\textbf{Reconstructions of a test sample from the Rain12 dataset.}  \label{fig:derain_comparison}}
\end{figure*}

\begin{table*}[t]
\normalsize
\centering
\normalsize
\begin{tabular}{ccccc}\toprule
 & \multicolumn{2}{c}{Rain12} & \multicolumn{2}{c}{Rain100L} \\ 
 \cmidrule(r){2-3}  \cmidrule(r){4-5}
& PSNR & SSIM & PSNR & SSIM \\ \midrule
LP \cite{Li2016rain}  & 32.02 & 0.91 & 29.11 & 0.88 \\
DDN \cite{Fu2017removing}  &  31.78 & 0.90 & 32.16 & 0.94 \\
DSC \cite{Luo2015removing}  & 30.02 & 0.87 & 24.16 & 0.87 \\
JORDER \cite{Yang2020deep} & 33.92 & 0.95 &  36.61 & 0.97 \\
PReNet \cite{Ren2019progressive} & \cellcolor{grey}~~36.69~~ & \bf 0.96 & \cellcolor{grey}~~37.10~~ & \bf 0.98 \\
LGM (single-scale CSC; \cite{Khatib2021learned}) & 35.46 & 0.95 & 34.07 & 0.96 \\ 
MUSC (multiscale CSC; ours)  & \bf ~~36.77~~  & \bf ~~0.96~~  & \bf~~37.25~~ & \bf ~~0.98~~  \\ 
\bottomrule
\end{tabular}    
\vspace*{1mm}
\caption{\textbf{Performance on the deraining test set.} Boldface indicates the best performance; second-best results are highlighted in grey. All results are collected from \cite{Ren2019progressive} and \cite{Khatib2021learned} except MUSC. \label{tab:deraining-results}}
\end{table*}

\section{Experiments \label{sec:experiments}}

We report the performance of MUSC on deraining, CT reconstruction, and MRI reconstruction tasks. The motivations for choosing these tasks are as follows. First, we note that single-scale \gls{csc} models have recently been applied to the deraining task, achieving performance slightly worse than state-of-the-art deep networks \cite{Khatib2021learned}; we thus aim to test the capability of our multiscale approach on the same task. We additionally choose CT and MRI reconstruction tasks as there exist challenging, large-scale, and up-to-date benchmark datasets for these tasks. Two such datasets that we use are the LoDoPaB-CT \cite{Leuschner2021LoDoPabCT} and the fastMRI \cite{Zbontar2018fastMRI}. An additional strength of these two datasets is that the model evaluation process is carefully controlled: The evaluation on the challenge fold (for LoDoPaB-CT) or the test fold (for fastMRI) is restricted through an online submission portal with the ground truth hidden from the public. As a result, overfitting to these evaluation folds is difficult and quantitative comparisons are transparent.

Throughout this section, we use the MUSC architecture whose encoder and decoder dictionaries are displayed in Figure~\ref{fig:unet_schematic}\textbf{b} and mathematically defined in Appendix \ref{appendix:decoder_dictionary_description}. Hyper-parameter choices for the experiments are provided in Appendix \ref{appendix:experiment_setup}. For each task, we use well-known \gls{cnn} models as baselines. We note that, for the CT and MRI reconstruction tasks, there are two major approaches to employ \glspl{cnn}. In the first, \emph{model-based} approach, one applies neural networks on raw measurement data (sinogram data in CT and k-space data in MRI) by embedding a task-dependent forward operator (the Radon transform for CT and the Fourier transform for MRI) into multiple layers or iterations of the network. Learning methods of this approach can be highly performant at the cost of being computationally expensive, especially during training, since one needs to apply the forward operator (and the adjoint of its derivative) repeatedly \cite{Leuschner2021quantitative}. In the second, \emph{model-free} approach, the (pseudoinverse of the) forward operator is used at most \emph{once} during data preprocessing and is never used during subsequent supervised training. These preprocessed images contain imaging artifacts. During supervised learning, one applies a \gls{cnn} directly on these preprocessed images. The proposed MUSC is in this sense a model-free approach and we compare it to model-free baselines. We note that in this case one does not need to know the forward operator at all. The leading model-free baseline \gls{cnn} methods in this approach are typically U-Net variants tuned to the task at hand. For a more thorough comparison, we also implemented the original U-Net architecture proposed in \cite{Ronneberger2015unet} (schematically illustrated in Figure~\ref{fig:unet_schematic}\textbf{a}) in these tasks as additional baselines. 

While model-free approaches perform somewhat worse than model-based ones, our purpose here is to show that a general-purpose multiscale convolutional model can perform as well as convolutional neural networks \textit{ceteris paribus}, rather than to propose state-of-the-art reconstruction algorithms for specific problems. This general-purpose approach further allows us to tackle structured denoising problems such as deraining where the forward operator is simply the identity.

\begin{figure*}[ht]
\centering
\includegraphics[width= \textwidth]{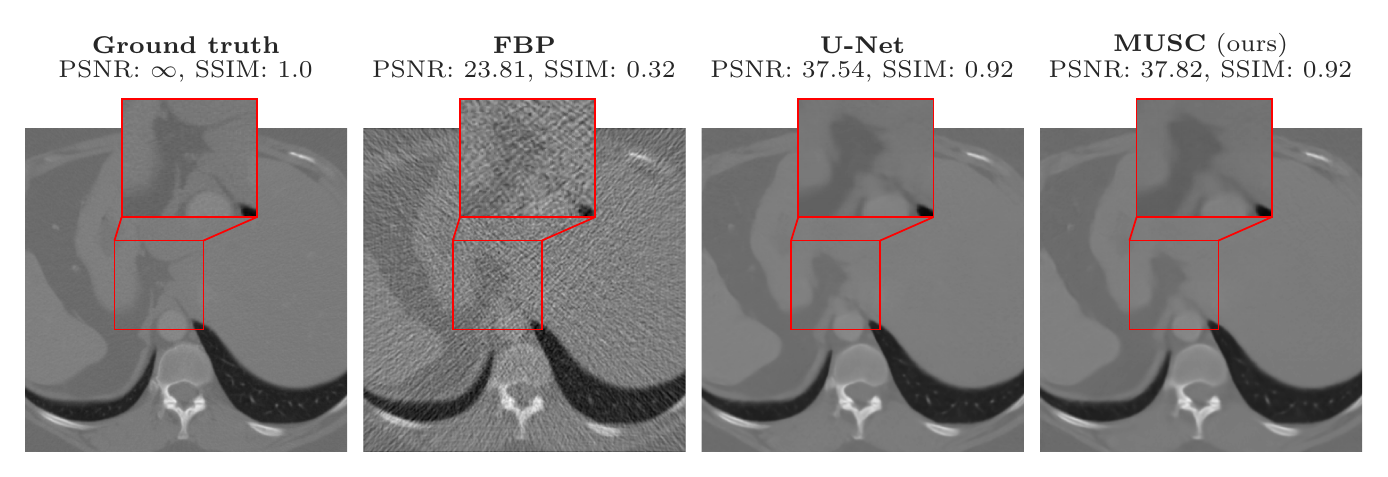} 
\caption{\textbf{Reconstructions of a test sample from the LoDoPaB-CT dataset.}  \label{fig:ct_comparison}}
\end{figure*}

\begin{table*}[t]
\normalsize
    \centering
    \begin{tabular}{lccccccc} \toprule
             & Number of parameters & PSNR    & PSNR-FR    & SSIM & SSIM-FR \\\toprule
        FBP  & -  & 30.19 $\pm$ 2.55 & 34.46 $\pm$ 2.18 & 0.727 $\pm$ 0.127 & 0.836 $\pm$ 0.085   \\
        TV   & - & 33.36 $\pm$ 2.74 & 37.63 $\pm$ 2.70 & 0.830 $\pm$ 0.121 & 0.903 $\pm$ 0.082  \\
        CINN & 6.43M & 35.54 $\pm$ 3.51 & 39.81 $\pm$ 3.48 & 0.854 $\pm$ 0.122 & 0.919 $\pm$ 0.081 \\
        U-Net++ & 9.17M  & 35.37 $\pm$ 3.36 & 39.64 $\pm$ 3.40 & 0.861 $\pm$ 0.119 & 0.923 $\pm$ 0.080 \\
        MS-D-CNN & 181.31K & 35.85 $\pm$ 3.60 & 40.12 $\pm$ 3.56 & 0.858 $\pm$ 0.122 & 0.921 $\pm$ 0.082 \\
        U-Net & 31.04M & 35.87 $\pm$ 3.59 & 40.14 $\pm$ 3.57 & 0.859 $\pm$ 0.121 &	0.922 $\pm$ 0.081 \\
        LoDoPaB U-Net &  613.32K & \cellcolor{grey}~~36.00 $\pm$ 3.63~~ & \cellcolor{grey}~~40.28 $\pm$ 3.59~~ & \cellcolor{grey}~~0.862 $\pm$ 0.119~~ & \cellcolor{grey}~~0.923 $\pm$ 0.079~~ \\
        MUSC (ours) & 13.87M  & \bf~~{36.08} $\pm$ 3.68~~ & \bf~~{40.35} $\pm$ {3.64}~~ & \bf~~{0.863} $\pm$ {0.119}~~ & \bf~~{0.924} $\pm$ \num[math-rm=\bf]{0.080}~~ \\  \bottomrule
    \end{tabular} 
	\vspace*{1mm}
    \caption{\textbf{Performance on the LoDoPaB-CT challenge set.} All values are taken from the official challenge leaderboard.}
    \label{tab:ct_challenge_results}
\end{table*}

\subsection{Deraining \label{subsec:de_raining}}
Image deraining aims to remove rain streaks from an image. Formally, a rainy image $\vz$ is expressed as $\vz = \vgt + \vs$, where $\vgt$ is a clean image and $\vs$ is the rain streaks component. The goal is to reconstruct the clean image $\vgt$ based on the rainy image $\vz$. Recently, single-scale \gls{csc} models have been applied to the draining task \cite{Khatib2021learned}. Despite theoretical progress, these single-scale \gls{csc} models still fall short competing with leading deep learning models, as remarked in \cite{Khatib2021learned}. In this section, we demonstrate that our multiscale \gls{csc} model closes this performance gap.

Throughout this subsection, we follow the experiment setup of \cite{Khatib2021learned}. We use 200 clean and rainy image pairs as the training dataset. A rainy image is created by adding synthesized rain streaks to its clean version. We use two test sets, Rain12 \cite{Li2016rain} and Rain100L \cite{Yang2020deep}, to benchmark our results.  Similar to \cite{Khatib2021learned}, we train our model to restore rain streaks based on rainy images; a derained image is then the difference between a rainy image and the restored rain streaks. To be consistent with \cite{Khatib2021learned, Yang2020deep, Ren2019progressive}, the evaluation result is calculated after transforming the image into the luma component in the YCbCr domain using the software provided by \cite{Ren2019progressive}. Additional details of the experiment are provided in Appendix \ref{appendix:experiment_setup}. 

We report the reconstruction performance in Table \ref{tab:deraining-results} and visualize the reconstruction results in Figure~\ref{fig:derain_comparison}. Our multiscale convolutional dictionary approach matches or outperforms baseline methods. Notably, it improves upon the LGM method (the single-scale CSC approach of \cite{Khatib2021learned}) by a non-trivial margin. 

\begin{figure*}[!htbp]
\centering
\includegraphics[width= \textwidth]{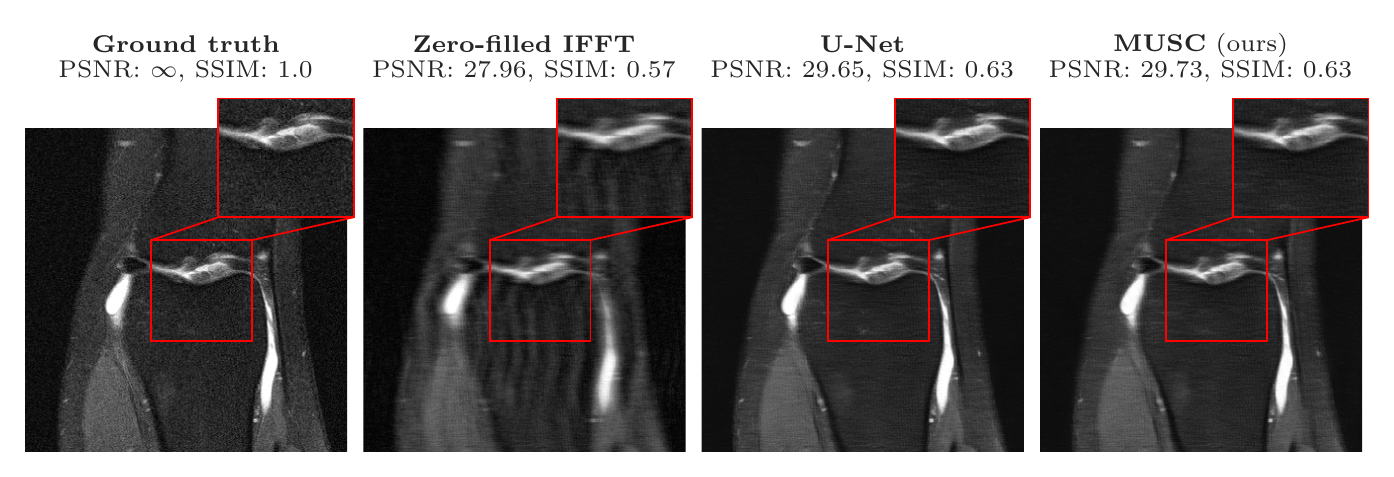} 
\caption{\textbf{Reconstructions of a test sample from the fastMRI single-coil knee dataset.}  \label{fig:fastmri_comparison}}
\end{figure*}

\begin{table*}[!htbp]
\normalsize
\centering
\begin{tabular}{cccccccc}\toprule
  & Number of parameters & \multicolumn{2}{c}{NMSE} & \multicolumn{2}{c}{PSNR   } & \multicolumn{2}{c}{SSIM} \\ 
 \cmidrule(r){3-4} \cmidrule(r){5-6} \cmidrule(r){7-8} 
 & & PD & PDFS & PD & PDFS & PD & PDFS \\ \midrule
 TV & - & 0.0287 &  0.0900    & 31.4  & 27.7  & 0.645 & 0.494    \\
 U-Net & 31.04M & 0.0161  & 0.0538    & 33.8  &  29.9  &  0.809  &   \cellcolor{grey}~~0.631~~   \\
 fastMRI U-Net-256 & 214.16M & \bf  0.0154 &   \bf  0.0525 & \bf  34.0 & \bf  30.0 & \bf 0.815 &    \bf 0.636 \\
 MUSC (ours) &  13.87M &  \cellcolor{grey}~~0.0156~~  &  \cellcolor{grey}~~0.0526~~ &   \bf  34.0 & \bf~~30.0~~  &  \cellcolor{grey}~~0.811~~ &  \cellcolor{grey}~~0.631~~      \\
\bottomrule
\end{tabular}
\vspace*{1mm}
\caption{ \textbf{Performance on fastMRI single-coil knee validation data}. Results are collected from \cite{Zbontar2018fastMRI} except U-Net and MUSC. The fastMRI U-Net-32 model refers to a U-Net variant defined in \cite{Zbontar2018fastMRI} whose output after the first convolution has 32 channels. Other models are defined similarly. PDFS and PD correspond to two MRI acquisition protocols with fat suppression (PDFS) and without fat suppression (PD) \cite{Zbontar2018fastMRI}.}
\label{tab:mri_validation_results}
\end{table*}

\begin{table*}[!htbp]
\normalsize
\centering
\begin{tabular}{cccccccc}\toprule
  & Number of parameters & \multicolumn{2}{c}{NMSE} & \multicolumn{2}{c}{PSNR} & \multicolumn{2}{c}{SSIM} \\ 
 \cmidrule(r){3-4} \cmidrule(r){5-6} \cmidrule(r){7-8} 
 & & PD & PDFS & PD & PDFS & PD & PDFS \\ \midrule
 Zero-filled IFFT & - & 0.0266 & 0.0597 & 32.0 & 29.2 & 0.7765 & 0.6045 \\
 TV & - & 0.0221  & 0.0716  & 33.1 & 28.5 & 0.7036  & 0.5096  \\
 U-Net & 31.04M &  0.0130 & 0.0430  & 35.3  & 30.6   & \num0.8513   & 0.6626  \\
 fastMRI U-Net-256 & 214.16M & \bf 0.0115  & \bf 0.0414 & \bf 35.9 & \bf 30.8  & \bf 0.8618 & \bf 0.6680 \\
 MUSC (ours) & 13.87M & \cellcolor{grey}~~0.0126~~   & \cellcolor{grey}~~0.0421~~  &  \cellcolor{grey}~~35.5~~  & \cellcolor{grey}~~30.7~~  & \cellcolor{grey}~~0.8537~~  & \cellcolor{grey}~~0.6633~~  \\
\bottomrule
\end{tabular} 
\vspace*{1mm}
\caption{\textbf{Performance on single-coil knee test data}. Results are collected from the fastMRI public leaderboard.}
\label{tab:mri_test_results}
\end{table*}

\subsection{CT reconstruction \label{subsec:ct_recon}}
Computed tomography (CT) aims to recover images from their sparse-view sinograms. We use the LoDoPaB-CT dataset \cite{Leuschner2021LoDoPabCT} to benchmark our results. This dataset contains more than 40000 pairs of human chest CT images and their simulated low photon count measurements. The ground truth images of this dataset are human chest CT scans corresponding to the LIDC/IDRI dataset \cite{Armato2011}, cropped to 362 × 362 pixels. The low-dose projections are simulated using the default setting of \cite{Leuschner2021LoDoPabCT}.

To train our MUSC, we use the default dataset split as recommended in \cite{Leuschner2021LoDoPabCT}: The dataset is divided into 35820 training samples, 3522 validation samples, 3553 test samples, and 3678 challenge samples. Here, the ground-truth samples from the challenge dataset are hidden from the public; the evaluation on this fold is performed through the online submission system of the LoDoPaB-CT challenge\footnote{\url{https://lodopab.grand-challenge.org/challenge/}}.

We compare the reconstruction performance of MUSCs with five modern \gls{cnn} baselines, namely CINN \cite{Denker2020conditional}, U-Net++ \cite{Zhou2018unet}, MS-D-CNN \cite{Pelt2018improving}, U-Net \cite{Ronneberger2015unet}, and LoDoPaB U-Net \cite{Leuschner2021LoDoPabCT}; the LoDoPaB U-Net refers to a modified U-Net architecture tailored to the LoDoPaB-CT task.  Figure~\ref{fig:ct_comparison} shows the reconstruction results of a test sample. In Table~\ref{tab:ct_challenge_results}, we quantitatively compare MUSC with two classic methods (FBP and TV) together with five \gls{cnn} baseline methods mentioned above. As shown in Table~\ref{tab:ct_challenge_results}, MUSC outperforms all baselines. The metrics PSNR and PSNR-FR are taken from \cite{Leuschner2021quantitative}: For a ground-truth signal $\vx^{\dagger}$ and its approximation $\widehat{\vx}$, we define
\[ \operatorname{PSNR}\left(\widehat{\vx},\vx^{\dagger}\right) \coloneqq 10 \log _{10}\left(\frac{\max (\vx^{\dagger}) - \min (\vx^{\dagger}) }{\operatorname{MSE}\left(\widehat{\vx},\vx^{\dagger}\right)}\right), \]
\[ \operatorname{PSNR-FR}\left(\widehat{\vx}, \vx^{\dagger} \right) \coloneqq 10 \log _{10}\left(\frac{1}{\operatorname{MSE}\left(\widehat{\vx},\vx^{\dagger}\right)}\right).\]

\begin{figure*}
\includegraphics[width= \textwidth]{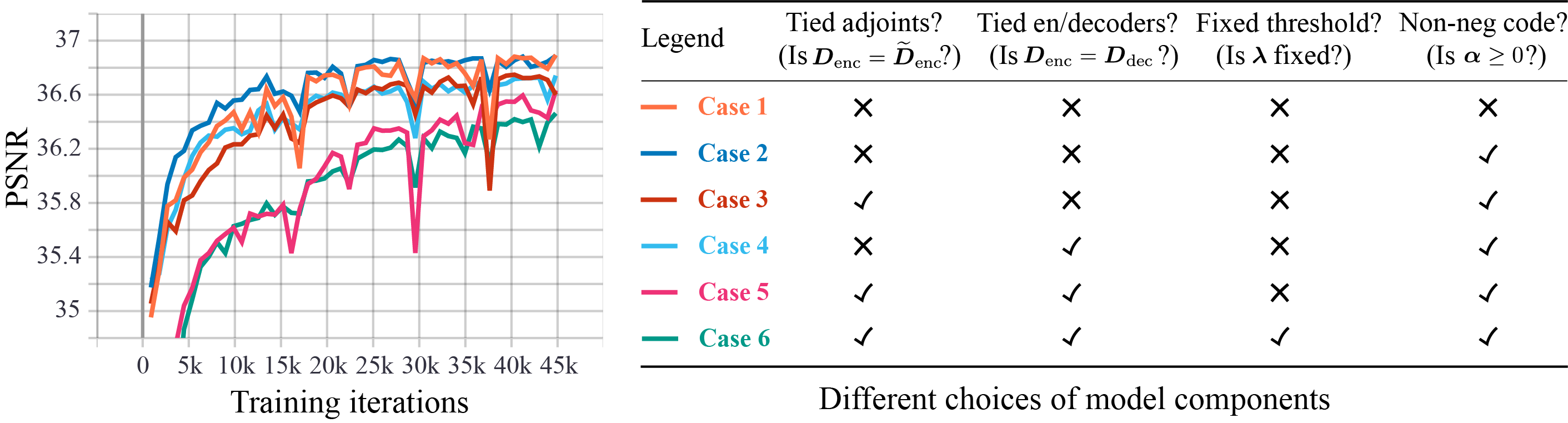}
\caption{\textbf{Ablation study on how different choices of model components affect the overall performance.} Left panel shows the PSNR (evaluated on validation samples) of six trained models as the training progresses; right panel shows the configuration of each trained model, where Case 2 corresponds to the usual setting used in other sections of this paper. For training, we used first 10\% of training samples of the LoDoPab-CT dataset; the validation samples are 50 samples in the validation fold of the dataset.}
 \label{fig:ablation}   
\end{figure*}

\begin{figure*}[!ht] 
\centering
\includegraphics[width= 0.49 \textwidth]{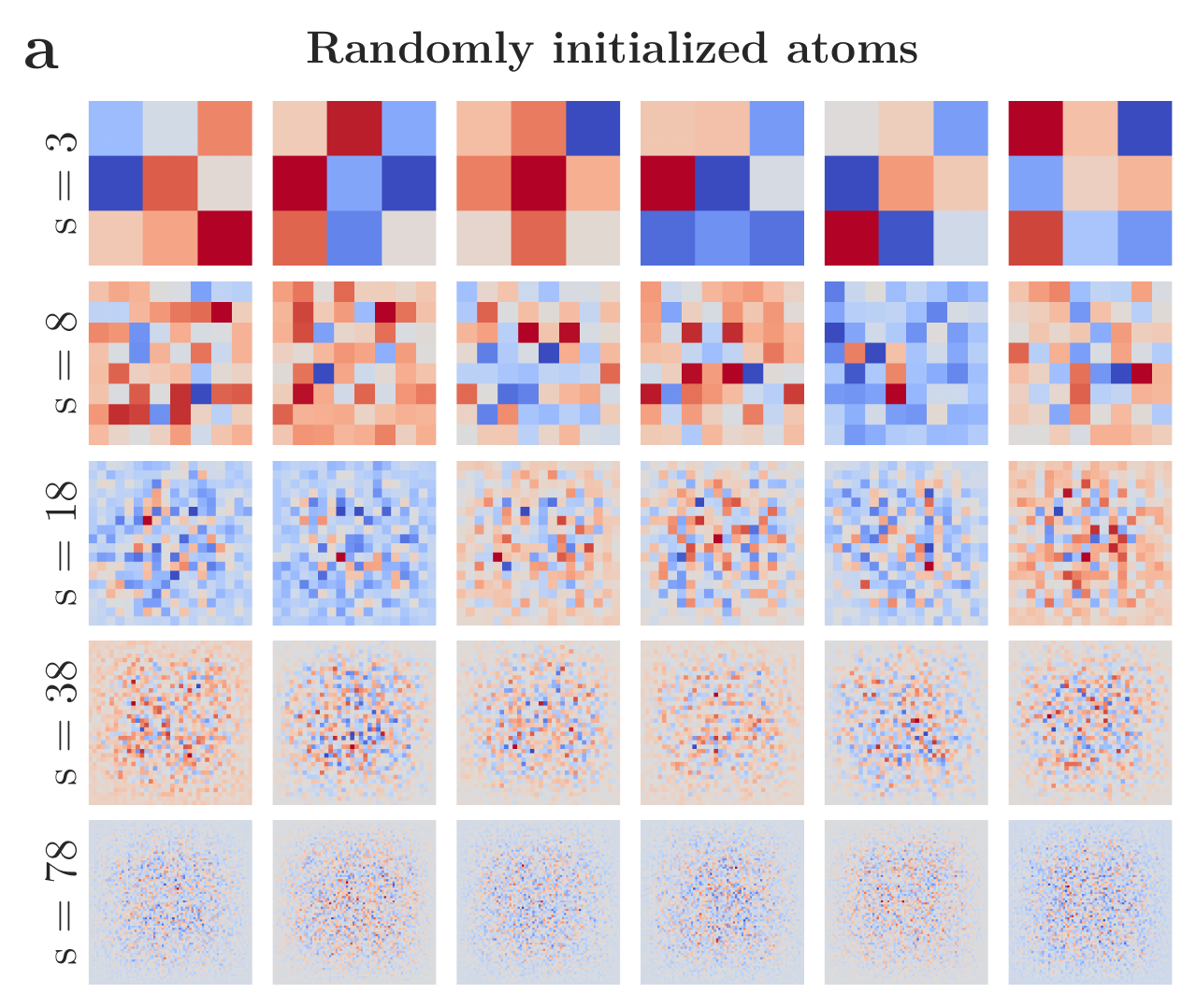} 
\includegraphics[width= 0.49 \textwidth]{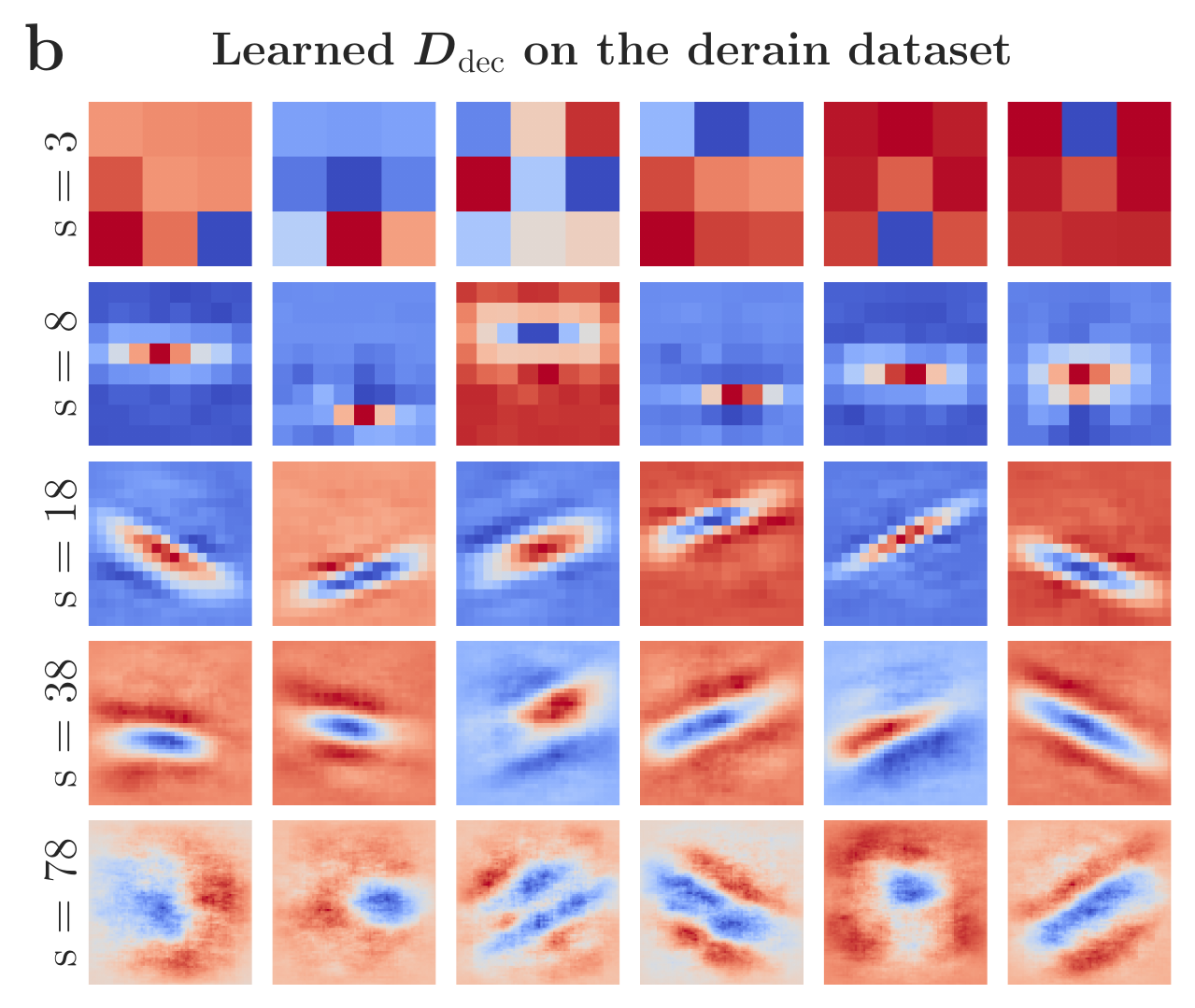} \\
\includegraphics[width= 0.49 \textwidth]{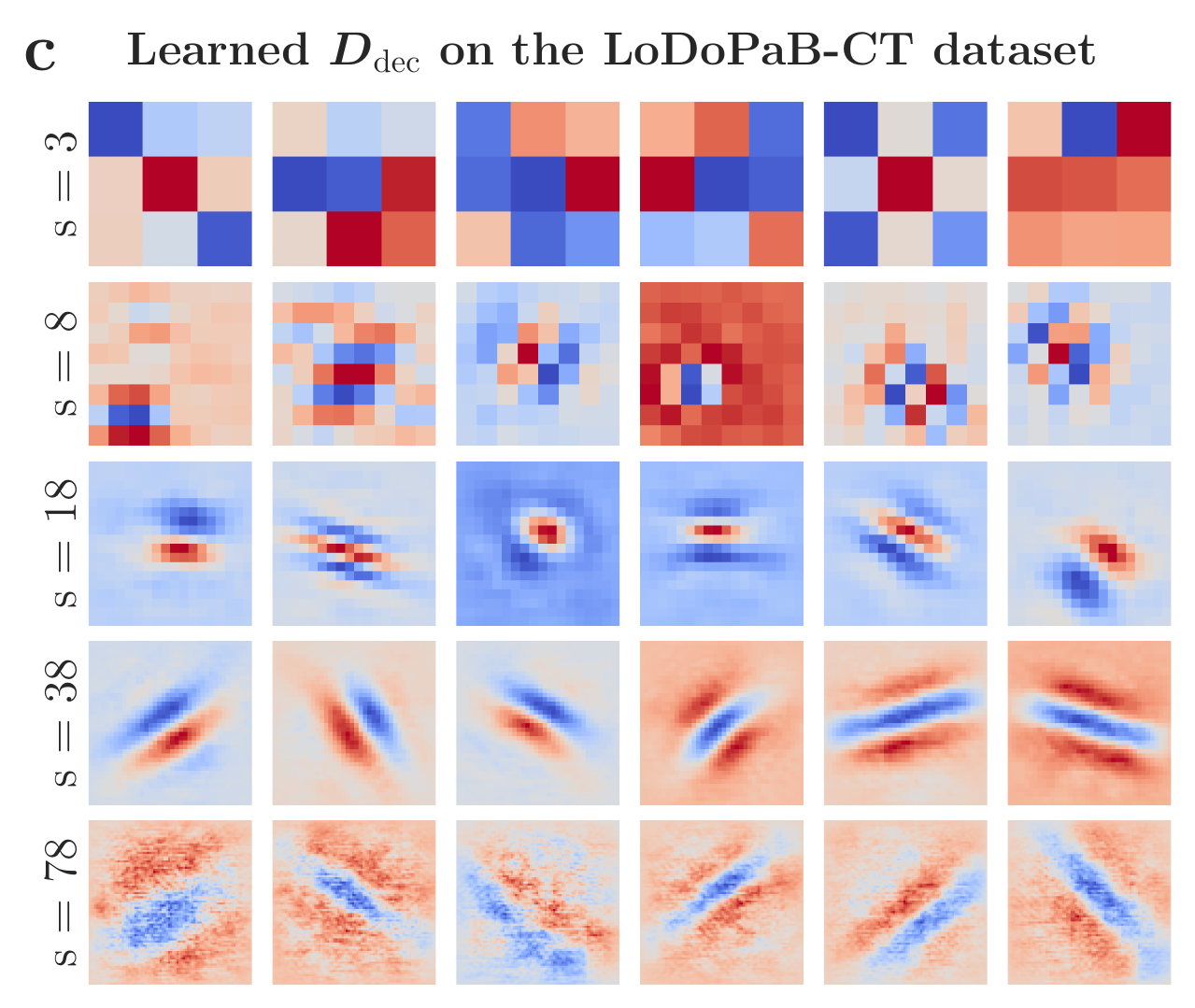}  
\includegraphics[width= 0.49 \textwidth]{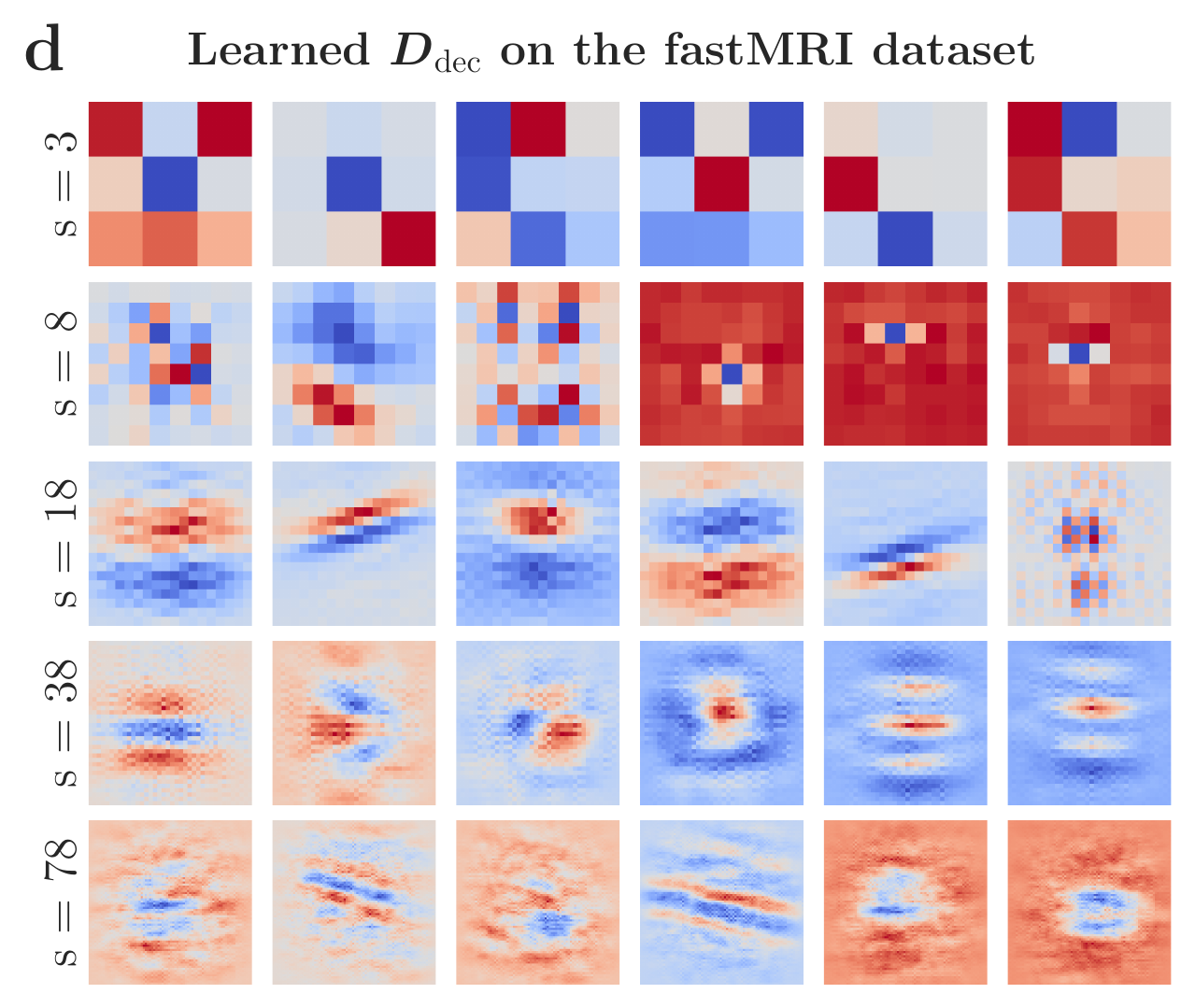} \\
\vspace*{0.5mm}
    \caption{\textbf{Atoms in a randomly initialized (panel~\textbf{a}) and learned decoder dictionary based on the derain dataset (panel~\textbf{b}), LoDoPab-CT (panel~\textbf{c}) and fastMRI (panel~\textbf{d}) dataset.} For all panels, each row corresponds to a support size (denoted by $s$) of atoms. Top rows are atoms that have a small support size; bottom rows are atoms that have a large support size. In each row, atoms are displayed in a sorted order of decreasing $\ell_2$ norm; for the visualization purpose, they are normalized into the range $[-1,1]$. \label{fig:fastmri_atoms}}
\end{figure*}

\subsection{MRI reconstruction}

We further considered the task of accelerated \gls{mri} reconstruction using the fast\gls{mri} dataset \cite{Zbontar2018fastMRI} procured by Facebook and NYU. Specifically, we used the single-coil knee dataset with a 4-fold acceleration factor. This dataset contains 973 volumes or 34742 slices in the training set, 199 volumes or 7135 slices in the validation set, and 108 volumes or 3903 slices in the test set. The ground-truth images in the test set are not provided to the public and the evaluation must be made through the fastMRI online submission system\footnote{\url{https://fastmri.org/}}. 

Following the training protocol of \cite{Zbontar2018fastMRI}, we first transformed the undersampled k-space measurements into the image space using zero-filled Inverse Fast Fourier Transform (IFFT); we use the transformed images as input to MUSC and other \gls{cnn} baselines. Consistent with previous work \cite{Zbontar2018fastMRI}, we found that U-Net variants deliver exceptional performance on validation samples (Table~\ref{tab:mri_validation_results}). Remarkably, MUSC performs on-par with U-Net variants, yielding visually indistinguishable results (Figure~\ref{fig:fastmri_comparison}). We next evaluate the U-Net and the MUSC on test samples through the fastMRI submission system. On the test data, the proposed MUSC produces results comparable to the best-performing U-Net result (fastMRI U-Net-256) provided by the fastMRI challenge organizer while having an order of magnitude fewer parameters (Table~\ref{tab:mri_test_results}).

\begin{figure*}[!ht]
    \centering
    \includegraphics[width= 0.49 \textwidth]{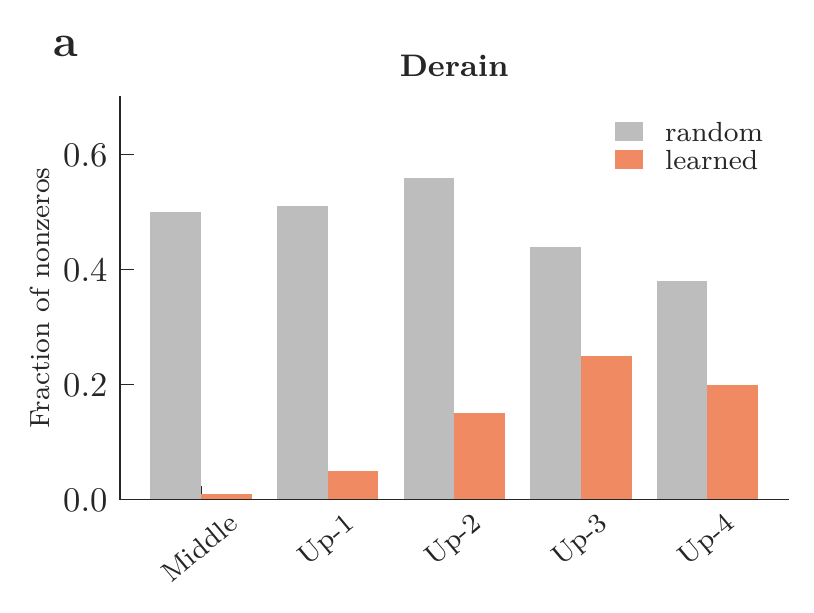}
    \includegraphics[width= 0.49 \textwidth]{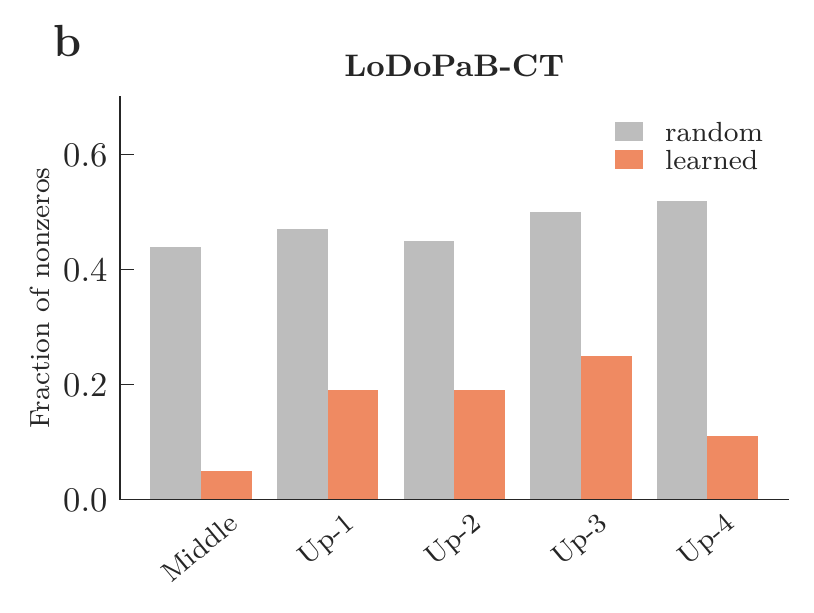} \\
    \includegraphics[width= 0.49 \textwidth]{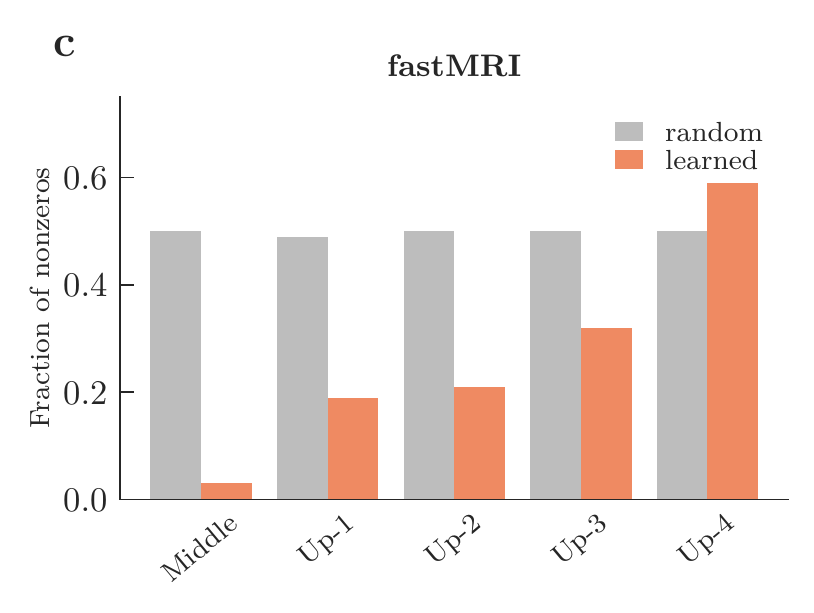}
    \includegraphics[width= 0.49 \textwidth]{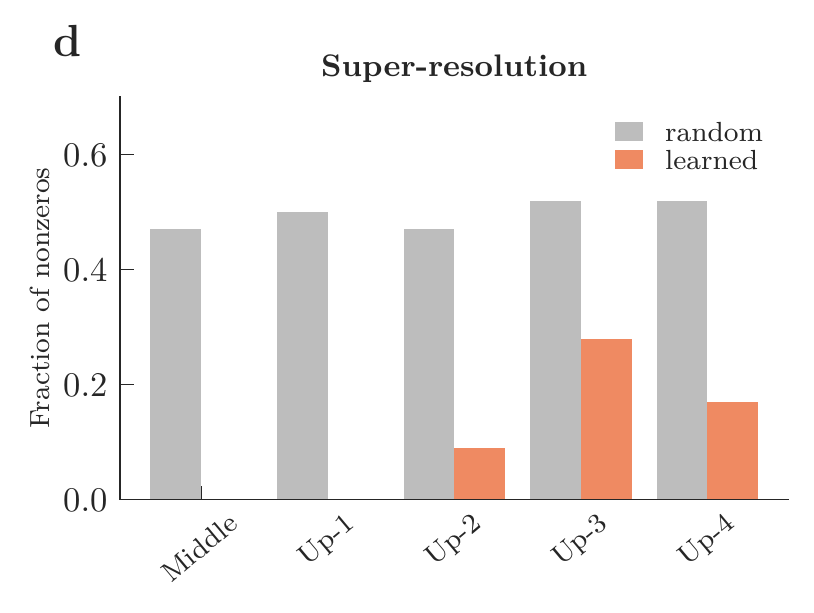}
    \caption{\textbf{Sparsity of dictionary-induced convolutional features maps.} Each bar corresponds to the sparsity level of a feature map tensor from the ``deepest'' activations corresponding to large-support atoms (``Middle'') to the ``shallowest' activations corresponding to small-support atoms (``Up-4'').}
    \label{fig:sparsity_fastmri}
\end{figure*}

\subsection{Single-image super-resolution}
We have additionally tested the MUSC on a standard super-resolution task, whose results are deferred to Appendix \ref{appendix:superresolution}. The goal of this task is to recover a high-resolution image from its degraded, low-resolution version. Unlike tasks such as CT and MRI reconstruction, in which the image degradation processes introduce long-range spatially correlated noise like streak artifacts, the blurring process in the super-resolution task is spatially local. 

In this case, we do not observe a performance gain of using a multiscale model -- either U-Net or MUSC -- over state-of-the-art single-scale \gls{csc} models. Interestingly, MUSC outperforms the U-Net, but is up to 0.5 dB worse than single-scale CSC.

In subsection \ref{subsection:probing}, we study this phenomenon by analyzing the sparse code yielded by MUSC. In the super-resolution tasks, the nonzeros in sparse codes are confined to high-resolution channels, or, equivalently, small filter supports which only leverage local information. This is well aligned with the intuition that the blurring forward operator mixes information only locally. It suggests that the right strategy is to use a large number of small-support filters just like CSC does, instead of ``wasting'' trainable parameters on unused large scales. We similarly find that a single-scale \gls{csc} model works better than MUSC on a denoising (Gaussian noise removal) task. Together, these findings suggest that multiscale features are no panacea for imaging inverse problems; the configuration of scales needs to resonate with the task-dependent forward operator that we aim to invert.

\subsection{Ablations on the choices of model components \label{subsec:ablation-study}} 
In Figure~\ref{fig:ablation}, we show ablation experiments that demonstrate how different choices of model components affect the overall performance. There, Case 2 is the off-the-shelf setup we have used in all other sections of this paper; this option has the fastest learning speed and highest end-point accuracy. Consistent with findings in \cite{Lecouat2020flexible, Lecouat2020fully, Zarka2020Deep}, we find it advantageous to use untied adjoints as described in \eqref{eq:relax-ista-execution}: Untied dictionaries (Cases 1, 2, and 4 in Figure~\ref{fig:ablation}) in general perform much better than {tied} dictionaries with $\encdict = \encdict = \decdict$ (Cases 5 and Case 6). What is more, we find that learnable threshold $\vlambda$ gives better results than fixed threshold. The non-negative constraint of sparse code $\valpha \geq 0$ does not greatly influence the end-point performance of models, although with the constraint the model learns slightly faster (Case 2) than without (Case 1).

\subsection{Probing multiscale dictionary-based representations \label{subsection:probing}}
Thus far, we have shown that our proposed multiscale \gls{csc} approach, dubbed MUSC, performs comparably to state-of-the-art \glspl{cnn} in a range of imaging inverse problems. This is noteworthy, as the strong performance is achieved simply by employing a multiscale dictionary -- as opposed to a single-scale one -- in an otherwise standard CSC paradigm. The strong performance suggests the usefulness of the multiscale representation. We now analyze our learned dictionaries and their induced sparse representations. %

\paragraph{Visualizing dictionary atoms}

We visualize dictionary atoms of the MUSC. To extract a dictionary atom from a dictionary $\mD$, we first prepare an \emph{indicator code} $\vdelta$, which is a collection of multichannel tensors that takes a value $1$ at a certain entry and $0$ elsewhere; a dictionary atom corresponds to that entry is computed as $\mD \vdelta$. Note that, different positions of the nonzero entry may give rise to atoms of different support sizes. This can be seen in Figure~\ref{fig:unet_schematic}\textbf{b}: The indicator code is illustrated as the grey boxes; depending on the grey box the nonzero entry resides in, the sparse code activates different receptive fields under composite convolutions and transposed convolutions. If the nonzero entry resides in the top-most box, then the support of the atom is 3 as it undergoes only a single 3 × 3 convolution; if the nonzero entry is in one of the lower boxes, the support of the atom is larger as the code undergoes multiple convolutions and one or more transposed convolutions. 

In Figure~\ref{fig:fastmri_atoms}, we show samples of multiscale atoms in $\decdict$ of varying sizes -- we crop these atoms to only show their nonzero support regions. As can be seen in Figure~\ref{fig:fastmri_atoms}\textbf{b}-\textbf{d}, the learned dictionaries contain Gabor-like or curvelet-like atoms with different spatial widths, resolutions, and orientations. Thus the learned dictionaries indeed exploit multiscale features. For comparison, we also show a randomly initialized dictionary (Figure~\ref{fig:fastmri_atoms}\textbf{a}). Unlike a learned dictionary, a random dictionary does not exhibit structures in atoms. We also visualize atoms of encoder dictionaries $\encdict$ and $\adjencdict$ in Appendix~\ref{subsection:ista-unet-other-dictionaries}. Using a similar technique, we also probe the multiscale representations learned by U-Nets in Appendix~\ref{subsection:probe_unet}.

\paragraph{Sparsity levels of representations}

We anticipate that the trained dictionaries induce different sparsity levels at different resolution levels in a task-dependent manner: More non-zeros associated with large-support atoms are useful when imaging artifacts have long-range correlations (e.g., streak artifacts in CT) than when the artifacts are localized (e.g., deraining or super-resolution). 

Figure~\ref{fig:sparsity_fastmri} shows the sparsity levels across tasks, both before and after dictionary learning. We observe that, prior to any learning, the sparsity levels induced by randomly initialized dictionaries (grey bars) are approximately uniform across scales. After learning, the sparsity levels of feature maps differentiate in a task-dependent way (orange bars in all panels).  This task-dependent differentiation suggests the usefulness of multiscale representations -- the learned sparsity levels are neither collapsed to a single scale nor remain uniform across spatial scales; instead, they are weighted and combined across scales in a problem-dependent way. A curious effect of multiscale learning arises in super-resolution (panel \textbf{d}): the activations are nonzero \emph{only in high-resolution features} (``Up-2'', ``Up-3'', and ``Up-4''), corroborating the intuition that low-resolution features are not important for this task. Additionally, comparing the ``Middle'' bars across panels, we see that CT and MRI reconstruction tasks indeed use more nonzero coefficients on large-support atoms than tasks such as deraining and super-resolution.

\section{Discussion}

The \gls{csc} paradigm provides a natural connection between sparse modeling and \glspl{cnn}. Despite being mathematical principled, existing \gls{csc} models still fall short competing with \glspl{cnn} in terms of empirical performance on challenging inverse problems. In this work, we report one simple and effective way to close the performance gap between \gls{csc} and \gls{cnn} models: incorporating a multiscale structure in the \gls{csc} dictionaries. Crucial to our approach is the structure of our constructed multiscale dictionary: It takes inspiration from and closely follows the highly successful U-Net model. We show that the constructed multiscale dictionary performs on par with leading \glspl{cnn} in major imaging inverse problems. These results suggest a strong link between dictionary learning and \glspl{cnn} -- in both cases, multiscale structures are essential ingredients.

Beyond empirical performance, we believe that the interpretability of the proposed MUSC is showing the way towards an interpretable deep learning model. An interpretable model consists of components whose objectives and functionality have nominal values. The MUSC fulfills this desideratum by incorporating two modules with well-understood objectives, one implementing sparse coding and the other dictionary-based synthesis. 

Overall, our work demonstrates the effectiveness and scalability of \gls{csc} models on imaging inverse problems. While deep neural networks are profoundly influencing image reconstruction, our work shows promise in a different direction: the principles of sparsity and multiscale representation developed decades ago are still useful in designing performant, parameter-efficient (compared to mainstream CNNs), and interpretable architectures that push the current limits of machine learning for imaging inverse problems.

\section*{Acknowledgments}
We thank Johannes Leuschner for helpful discussions regarding the LoDoPaB-CT dataset.  Numerical experiments were partly performed at the sciCORE (\url{http://scicore.unibas.ch/}) scientific computing center at the University of Basel. We thank all anonymous reviewers whose comments helped improve this manuscript.

\newpage
\bibliographystyle{IEEEtran}
\bibliography{IEEEabrv,ista_unet}

\clearpage

\begin{appendices}

\section{The definition of the MUSC dictionary \label{appendix:decoder_dictionary_description}}

In the main text, we illustrate the architecture of a U-Net (Figure~\ref{fig:unet_schematic}\textbf{a}) and the corresponding MUSC decoder dictionary (Figure~\ref{fig:unet_schematic}\textbf{b}). Loosely speaking, the decoder dictionary is the decoding branch of a standard U-Net with all ReLU activations, batch normalization, and some convolution operations removed. Here we provide a formal definition of the decoder dictionary. While in the main text we assume for simplicty that signals are 1D vectors, in this section we represent RGB images as multichannel tensors. To that end, we first consider \gls{siso} operations whose input and output are 2D signals having a single channel. We follow by considering \gls{mimo} operations whose input and output are 3D tensors having multiple channels.

\paragraph{SISO convolution and transposed convolution}
Let $\vxi \in \RR^{H \times W}$ be a 2D signal with height $H$ and width $W$. We regard $\vxi$ as a function defined on the discrete domain $\{1, \ldots, H \} \times \{1, \ldots, W \}$. With standard zero padding, we extend the domain of $\vxi$ to $\ZZ \times \ZZ$. The notation $\vxi[i, j]$ represents the value of the function $\vxi$ at the coordinate $(i, j)$.  

\paragraph{SISO convolution} Given a 2D signal $\vxi \in \RR^{H \times W}$ and parameters (filter weights, filter impulse response) $\vw \in \RR^{3 \times 3}$, the convolution of $\vxi$ and $\vw$ is defined as
\begin{equation}%
    (\vw \ast \vxi)[i, j] = \sum_{p=-1}^{1} \sum_{q=-1}^{1} \vxi[i+p, j+q] \cdot\vw[p, q].
\end{equation}

\paragraph{SISO transposed convolution}

A transposed convolution in the U-Net consists of a bed-of-nails upsampling by a factor of two followed by a 2-by-2 convolution with an interpolating filter. The bed-of-nails upsampling interleaves zeros between samples, which can be written as $\vxi \otimes \upmat$ with $\otimes$ denoting the Kronecker product. The 2-by-2 convolution between the filter $\vv \in \RR^{2 \times 2}$ and signal $\vxi \in \RR^{H \times W}$ is written as
\begin{equation}%
    (\vv \circledast \vxi)[i, j] = \sum_{p=0}^{1} \sum_{q=0}^{1} \vxi[i+p, j+q] \cdot\vv[p, q].
\end{equation}

Putting these together, we have the following definition for a transposed convolution. %
\begin{defn}
A transposed convolution $\Ucal_{\vv}$ parameterized by the convolutional filter $\vv \in \RR^{2 \times 2}$ is a map defined by
\begin{align*}
  \Ucal_{\vv} \colon  \RR^{n \times n } &\to \RR^{2n \times 2n}  \\
 \vxi & \mapsto \vv \circledast (\vxi \otimes \upmat ).
\end{align*}
\end{defn}

The following technical result simplifies the upsampling calculation.
 \begin{prop} \label{prop:circledast_id} %
 Let $\vv \in \RR^{2 \times 2}$ be a matrix and let $\bar{\vv}$ be its vertically and horizontally flipped version, i.e., $\bar{\vv}[i, j] = \vv[3 - i, 3 - j]$ for $(i, j) \in \{1, 2\} \times \{1, 2\}$. Let $\vxi \in \RR^{n \times n}$ be a  matrix with dimension $n \geq 2$. Then the following equation holds:
 \begin{equation}
    \Ucal_{\vv} (\vxi) = \vxi \otimes \bar{\vv}. \label{eq:circledast_id}
 \end{equation}
 \end{prop}
 \begin{proof}
 
 Observe that the result of LHS and RHS are both $2n$-by-$2n$ matrices, each composed of $n \times n$ number of 2-by-2 submatrices. It is enough to show that each $2$-by-$2$ sub-matrix of the LHS equals that of the RHS if they have the same position indices. 
 
 We first expand the RHS. By definition of Kronecker product, we have
 \[ \vxi \otimes \bar{\vv} ={\begin{bmatrix} \vxi[1, 1] \bar{\vv} & \cdots & \vxi[1, n]  \bar{\vv} \\\vdots &  \ddots & \vdots \\  \vxi[n, 1] \bar{\vv} &  \cdots &  \vxi[n, n] \bar{\vv} \end{bmatrix}} \in \RR^{2n \times 2n}, \]
 where the $(i, j)$-th 2-by-2 sub-matrix of $\vxi \otimes \bar{\vv}$ has the form
\begin{equation} \label{eq:circledast_id_rhs}
\vxi[i, j] \bar{\vv}  =  \begin{bmatrix} \vxi[i, j] \cdot {\vv} [2,2] & \vxi[i, j] \cdot {\vv} [2,1] \\ \vxi[i, j] \cdot {\vv} [1,2] & \vxi[i, j] \cdot {\vv} [1,1] \end{bmatrix}, 
\end{equation}
for each $(i, j) \in \{1, \cdots, n \} \times \{1, \cdots, n \}$. 
 
For the LHS, we have $\Ucal_{\vv}(\vxi) = \vv \circledast (\vxi \otimes \upmat )$ by definition. Writing out the Kronecker product, we see that the $(i, j)$-th 2-by-2 submatrix of the LHS has the form
\begin{equation} \label{eq:circledast_id_lhs}
\vv \circledast \begin{bmatrix} 0 & 0 & 0  \\ 0 & \vxi[i, j] & 0 \\  0 & 0 & 0 \end{bmatrix}  = \begin{bmatrix} \vxi[i, j] \cdot \vv[2,2] & \vxi[i, j] \cdot \vv[2,1] \\ \vxi[i, j] \cdot \vv[1,2] & \vxi[i, j] \cdot \vv[1,1] \end{bmatrix},
\end{equation}
for each $(i, j) \in \{1, \cdots, n \} \times \{1, \cdots, n \}$. 
Comparing Equation~\eqref{eq:circledast_id_rhs} and Equation~\eqref{eq:circledast_id_lhs}, we see the each 2-by-2 sub-matrix of the LHS and RHS of Equation~\eqref{eq:circledast_id} are identical as claimed.
 \end{proof}
 
\paragraph{MIMO convolution and transposed convolution}

To extend our formulation to the \gls{mimo} case, we denote by $\vx = ( \vxi_1, \ldots, \vxi_C )  \in \RR^{C \times H \times W}$ a multi-channel signal. 

\paragraph{MIMO convolution} We let $\mW = ( \vw_{11}, \ldots,  \vw_{1C},  \ldots, \vw_{M1}, \ldots, \vw_{MC} ) \in \RR^{M \times C \times 3 \times 3}$ be a multi-channel convolutional kernel of 3-by-3 filters, where $C$ denotes the number of input channels and $M$ denotes the number of output channels. The MIMO convolution is defined as 
\[ \conv(\mW, \vx) = \sum_{c=1}^C \left(\vw_{1 c} \ast \boldsymbol{\xi}_{c}, \ldots, \vw_{M c} \ast \boldsymbol{\xi}_{c}\right) \in \mathbb{R}^{M \times H \times W}. \]

\paragraph{MIMO transposed convolution} Let $\vx = ( \vxi_1, \ldots, \vxi_C )  \in \RR^{C \times H \times W}$ be a multi-channel signal. We let $\mV = ( \vv_{11}, \ldots,  \vv_{1C},  \ldots, \vv_{M1}, \ldots, \vv_{MC} ) \in \RR^{M \times C \times 2 \times 2}$ be a multi-channel convolutional kernel of 2-by-2 filters. The MIMO transposed convolution is defined as 

\[ \tconv(\mV, \vx) = \sum_{c=1}^C \left( \vxi_{c} \otimes \bar{\vv}_{1c}, \ldots, \vxi_{c} \otimes \bar{\vv}_{Mc} \right) \in \mathbb{R}^{M \times 2H \times 2W}. \]

\paragraph{Up-block}
An up-block receives multichannel input from (i) low-resolution features in the shape of $(2C, H, W)$ and (ii) skip-connection features in the shape of $(C, 2H, 2W)$. With these input features, an up-block applies a transposed convolution to low-resolution features that halves the number of channels, concatenates the resulting features with skip-connection features, and then submits the concatenated results for another convolution. This can be written as 

\begin{footnotesize}
\begin{align*}
  \text{Up}_{\mW, \mV}: \RR^{C \times 2H \times 2W} \times \RR^{2C \times H \times W}   & \to \RR^{C \times 2H \times 2W} \\
 (\vbeta, \vx) & \mapsto \conv( [\vbeta ; \tconv( \vx, \mV) ], \mW),
\end{align*}
\end{footnotesize} where $\mW \in \RR^{2C \times C \times 3 \times 3}$ and $\mV \in \RR^{2C \times C \times 2 \times 2}$ are parameters and $[\cdot ; \cdot]$ denotes channel-wise concatenation.

\paragraph{The multiscale dictionary}
We now define the multiscale dictionary
\begin{align}
\mD : \RR^{ C \times H \times W} \times \cdots \times \RR^{ C/2^4 \times 2^4 \cdot H \times 2^4 \cdot W} & \to \RR^{C^{\textrm{out}} \times 2^4 \cdot H \times 2^4 \cdot W} \notag \\
(\valpha_0, \ldots, \valpha_4) & \mapsto \mD (\valpha_0, \ldots, \valpha_4)
\end{align}
that is used as encoders and decoders in our work. Here, $C^{\textrm{out} } = 1 \textrm{~or~} 3$ for grayscale and RGB images. We construct the dictionary $\mD$ as a linear transform $\mR$ that is independent of $C^{\textrm{out} }$ followed by a one-by-one convolution that produce $C^{\textrm{out}}$ channels:
\begin{equation}
\mD = \texttt{Conv1x1} \circ \mR.  
\end{equation}
The linear map $\mR$ is defined by cascading $4$ Up-blocks with parameters $\mW_i$ and $\mV_i, i = 1, \ldots, 4$:
\begin{align}
\mR : \RR^{ C \times H \times W} \times \cdots \times \RR^{ C/2^4 \times 2^4 \cdot H \times 2^4 \cdot W} & \to \RR^{ C/2^4 \times 2^4 \cdot H \times 2^4 \cdot W} \notag \\
 (\valpha_0, \cdots,  \valpha_4 ) & \mapsto \vxi_4,
\end{align} %
where 
\[\vxi_i = \text{Up}_{\mW_i, \mV_i} (\valpha_i, \vxi_{i-1}) \]
for $i = 1, \ldots 4$ with $\vxi_0 = \valpha_0$. That is, the function $\mR$ transforms multiscale sparse code $(\valpha_0, \ldots, \valpha_4)$ to a tensor $\vxi_4$ of shape $\RR^{ C/2^4 \times 2^4 \cdot H \times 2^4 \cdot W}$. The $1\times 1$ convolution operator $\mC$ next synthesizes these features into a tensor with channel number 1 or 3:
\begin{equation}
\texttt{Conv1x1}:  \RR^{ C/2^4 \times 2^4 \cdot H \times 2^4 \cdot W} \to \RR^{C^{\textrm{out}} \times 2^4 \cdot H \times 2^4 \cdot W}.
\end{equation}

\section{Can a learned sparse coder yield dense outputs?}
In the main text, we introduced the $K$-fold ISTA algoritm for sparse coding. The sparse coding result reads $\ista_{K} (\vz; \encdict, \vlambda ) $, where $\vz$ is a given input, $\encdict$ is an encoding dictionary, and $\vlambda$ is a threshold tensor. Since we also learn a $\vlambda$ from data during training, one concern is whether a learned $\vlambda$ be zero. Indeed, if this were to occur, the sparse coding results would be dense and it would defeat the purpose of sparse coding.

In this section, we argue that $\vlambda = 0$ is (i) not optimal for many inverse problems and (ii) unlikely to be learned from data. To see the non-optimality of $\vlambda = 0$, we observe that with the iteration number $K$ is large enough, with $\vlambda = 0$, and with zero initialization of sparse code, ISTA yields the smallest $\ell^2$-norm solution which can be written via the Moore--Penrose pseudoinverse,
$$
    \lim_{K \to \infty} \ista_{K}(\vz; \encdict, \mat{0}) = \encdict^+ \vz.
$$
Our task-driven dictionary learning problem thus has the form 
$$
\minimize_{\encdict,\decdict } \quad  \frac{1}{2M} \sum_{i = 1}^M \| \decdict \encdict^+ \valpha_{\vz_i} - \vgt_i  \|_2^2.
$$
This implies that $\lambda = 0$ results in a linear reconstruction method which comes with all the known drawbacks of linear methods. In particular, it cannot do better than the linear minimum mean square error (LMMSE) estimator (a generalized Wiener filter). Since regularization in ill-posed inverse problems entails the use of data models, and most useful data models are nonlinear (e.g., natural and medical images are known to be sparse or compressible in wavelet frames, but they do not belong to any linear subspace), these problems demand $\vlambda > 0$.

One may still wonder whether our learning procedure will overfit finite datasets with $\vlambda = 0$. To see that this will not happen, note that $\decdict$ and $\encdict$ are constrained to have a specific  structure: They are multiscale variants of block-Toeplitz matrices. Additionally, filters at high resolutions are convolutions of filters at lower resolutions which induces a rather complicated algebraic structure. As a result, the set of valid dictionaries in our model has a much smaller dimension than the set of all possible dictionaries of correct size, and solving \eqref{eq:ista_unet_objective} cannot be reduced to finding two generic overcomplete dictionaries that ``overfit'' to training data to achieve zero loss. In fact, forcing $\vlambda = 0$ typically incurs a large loss in \eqref{eq:ista_unet_objective}, so $\vlambda = 0$ is unlikely to be learned from data. This is the magic of multiscale convolutional sparsity.

\section{Power iteration \label{appendix:power_iter}}
We describe how to approximate the dominant eigenvalue of the matrix $\encdict^\top \encdict$ using by power iteration. We achieve this by first estimating the eigenvector associated with the dominant eigenvalue using the power iteration method, by recursively calculating
\begin{equation}
\vb_{k+1}=\frac{ \encdict^\top  \encdict \vb_k   }{\left\| \encdict^\top  \encdict \vb_k \right \|_2}
\end{equation}
up to some step $K$ with $\vb_0$ being an all-zero vector. The estimated dominant eigenvalue can then be derived from 

\begin{equation}
\lambda_{K}= \frac{ \vb_{K}^\top  \encdict^\top \encdict \vb_{K} }{\vb_{K}^\top \vb_K }.
\end{equation}

\section{Details of the experimental setup \label{appendix:experiment_setup}}

\begin{table*}[!htbp]
\centering
\begin{tabular}{cccccccccc}\toprule
 &   \multicolumn{3}{c}{Conv $3 \times 3$}   & \multicolumn{3}{c}{Trans-Conv $2 \times 2$}  & \multicolumn{3}{c}{Trans-Conv $1 \times 1$} \\ 
  \cmidrule(r){2-4} \cmidrule(r){5-7} \cmidrule(r){8-10} 
& in-channels & out-channels & stride &   in-channels & out-channels & stride & in-channels & out-channels & stride \\ \midrule
Scale 1 &  512  & 512  & 1 & 512 & 256 & 2 & -- & -- & --  \\
Scale 2 &  512  & 256  & 1 & 256 & 128 & 2 & -- & -- & --  \\
Scale 3 &  256  & 128  & 1 & 128 & 64 & 2 & -- & -- & --  \\
Scale 4 &  128  & 64  & 1 & 64 & 32 & 2 & -- & -- &  -- \\
Scale 5 &  64  & 32  & 1 & -- & -- & -- & 32 & 1 or 3 & 1   \\
\bottomrule
\end{tabular}
\vspace*{1mm}
\caption{ Parameters used at each scale of the encoder and decoder dictionaries. Scale 1 corresponds to the low-resolution scale (the bottom-most gray box in Figure 1\textbf{b}) and Scale 5 correspond to the high-resolution scale (the top-most gray box in Figure 1\textbf{b}). \label{fig:parameter-dictionary} }
\end{table*}

In Table~\ref{fig:parameter-dictionary}, we summarize parameters used in each scale of our dictionaries. In Table~\ref{tab:ista_unet_hyperparams} and Table~\ref{tab:classic_unet_hyperparams}, we summarize the hyperparameters used for training MUSC and the U-Net baseline. 

\begin{table}[h]
\centering
\begin{tabular}{lcccccc} \toprule
Task                                        & \shortstack{Epoch \\ number} & \shortstack{Batch \\ size}   & \shortstack{Learning \\ rate }  & \shortstack{Num \\ ISTA steps} & \shortstack{Lasso \\ parameter}  \\ \midrule
Derain                                      & 20000                      &16                &5e-4                & 5      &0.001         \\
LoDoPaB-CT                                  &70                          &2                 &2e-4                 &5       &0.001         \\
fastMRI                                     &70                          &2                   &5e-5                &5       &0.001            \\
Super-resolution                            &1500                          &2                   &5e-5                &5       &0.001            \\
\bottomrule
\end{tabular}
\vspace*{0.2cm}
\caption{\textbf{Hyperparameters used for the MUSC.} }
\label{tab:ista_unet_hyperparams}
\end{table}
For Derain and super-resolution task, we randomly crop images into $128 \times 128$ patches during gradient-descent training.

\begin{table}[H]
\normalsize
\centering
\footnotesize
\begin{tabular}{lccccccc} \toprule
Task                                         & \shortstack{Epoch \ ber} & \shortstack{Batch \\ size} &  \shortstack{Learning \\ rate}  \\ \midrule
LoDoPaB-CT                                   &70                        &32              &   0.001                     \\
fastMRI                                      &70                        &8               &   0.001               \\
\bottomrule
\end{tabular}
\vspace*{0.2cm}
\caption{\textbf{Hyperparameters used for the U-Net.} }
\label{tab:classic_unet_hyperparams}
\end{table}

Note that, for CT and MRI tasks, U-Nets and MUSCs are trained on full images; in deraining and super-resolution tasks,  U-Nets and MUSCs are trained on cropped images, where sizes of the cropped images are in Table \ref{tab:classic_unet_hyperparams}. Since in the deraining and LoDoPaB-CT tasks the range of the target images is non-negative, we clip the negative values of the synthesized image. During model training, we use weight normalization \cite{Salimans2016weight}, a reparametrization trick that decouples the magnitude of a convolutional filter from its direction. To enforce the positivity of the ISTA parameter $\vlambda,$ we reparametrize $\vlambda = \textrm{ReLU}(\tilde{\vlambda}) +$1e-5 and perform gradient-based learning on $\tilde{\vlambda}$ instead.

\section{Visualizing the atoms in encoder dictionaries \label{subsection:ista-unet-other-dictionaries}}

Consistent to how we visualize decoder dictionary atoms in Section \ref{subsection:probing}, we visualize atoms from/ encoder dictionaries $\encdict$ and $\adjencdict$ in Figure \ref{fig:encoder-atoms} and Figure \ref{fig:adj-encoder-atoms}. 
\begin{figure*}[!ht] 
\centering
\includegraphics[width= 0.49 \textwidth]{figs/rand_dictionary.pdf} 
\includegraphics[width= 0.49 \textwidth]{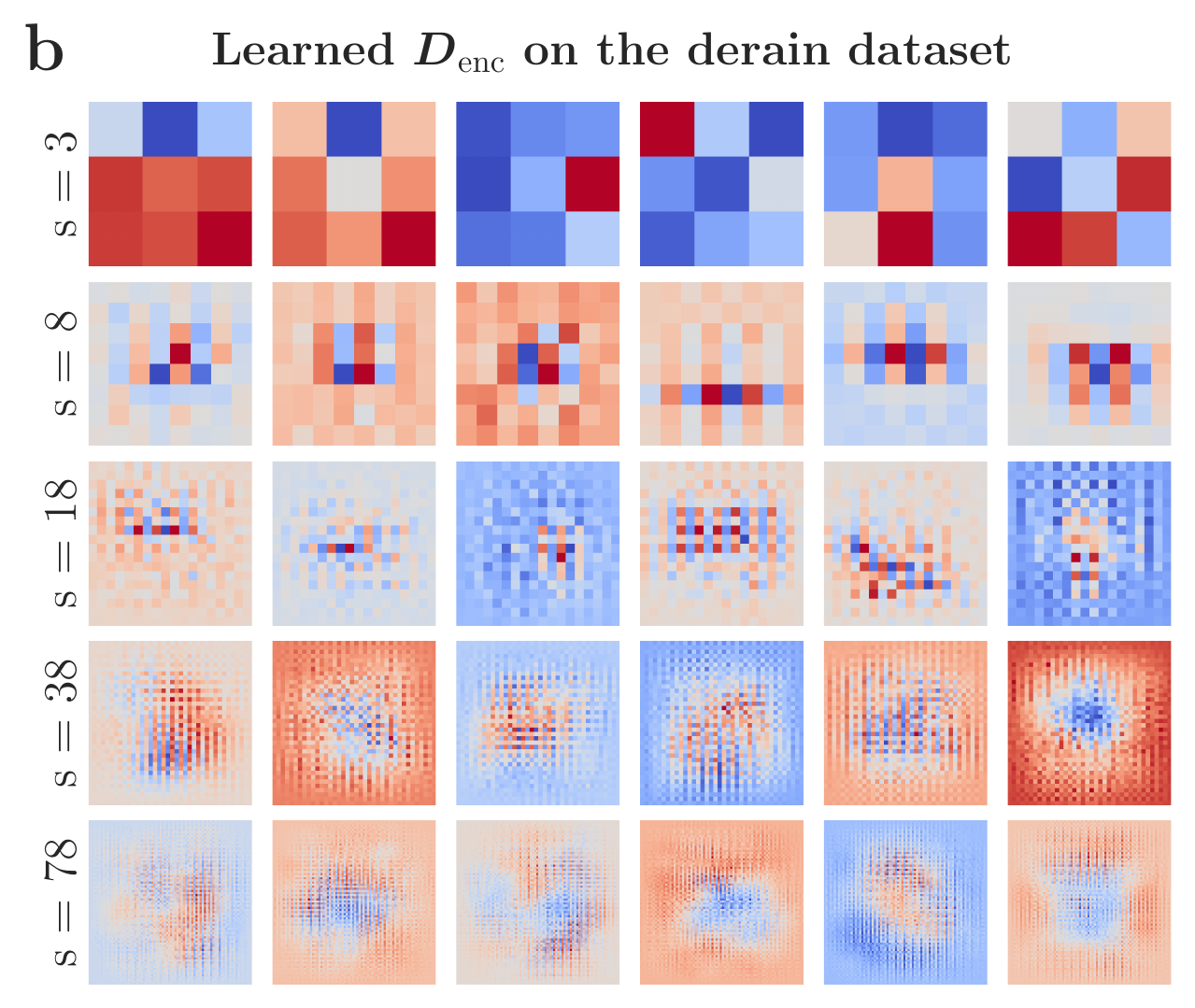} \\
\includegraphics[width= 0.49 \textwidth]{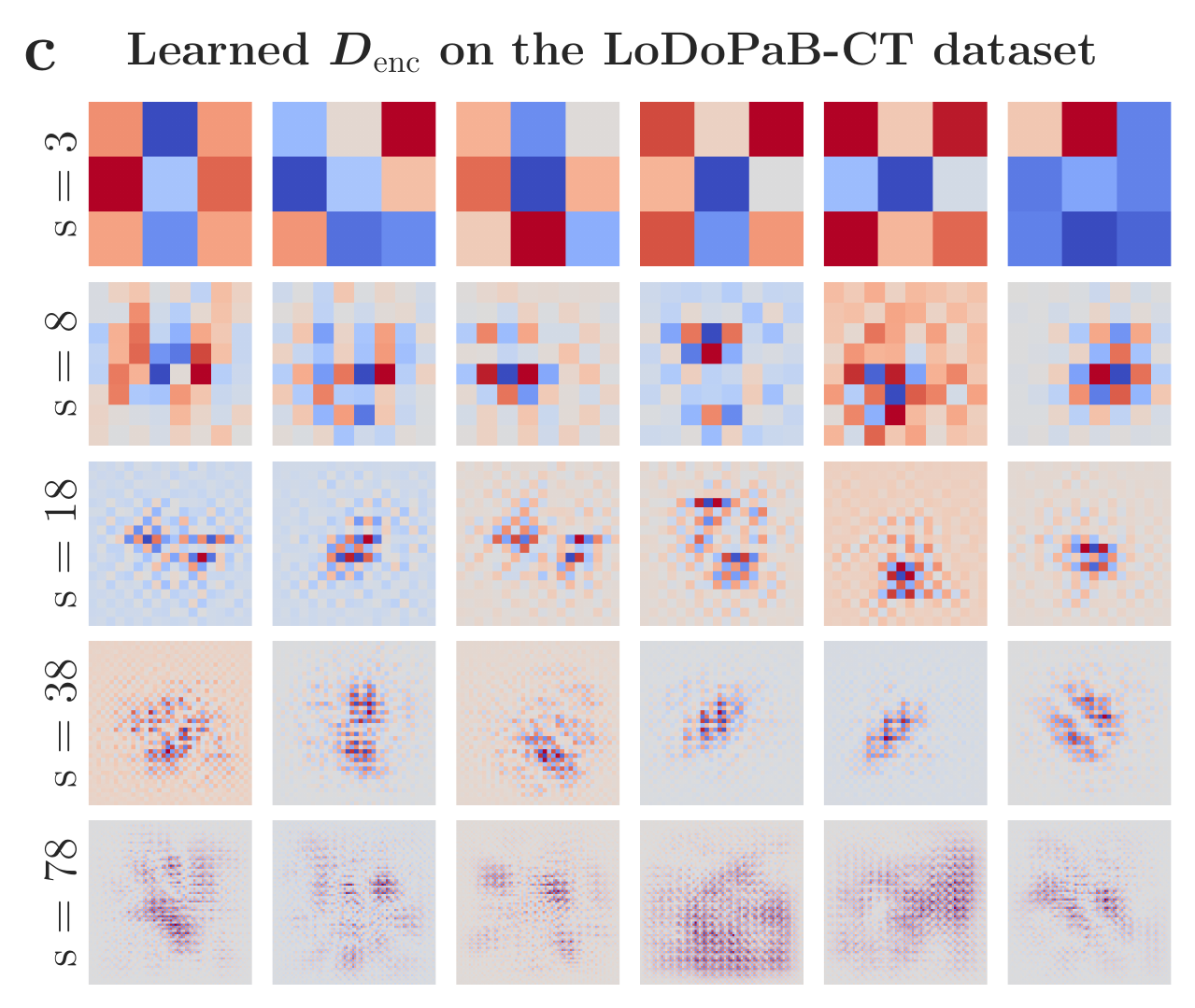}  
\includegraphics[width= 0.49 \textwidth]{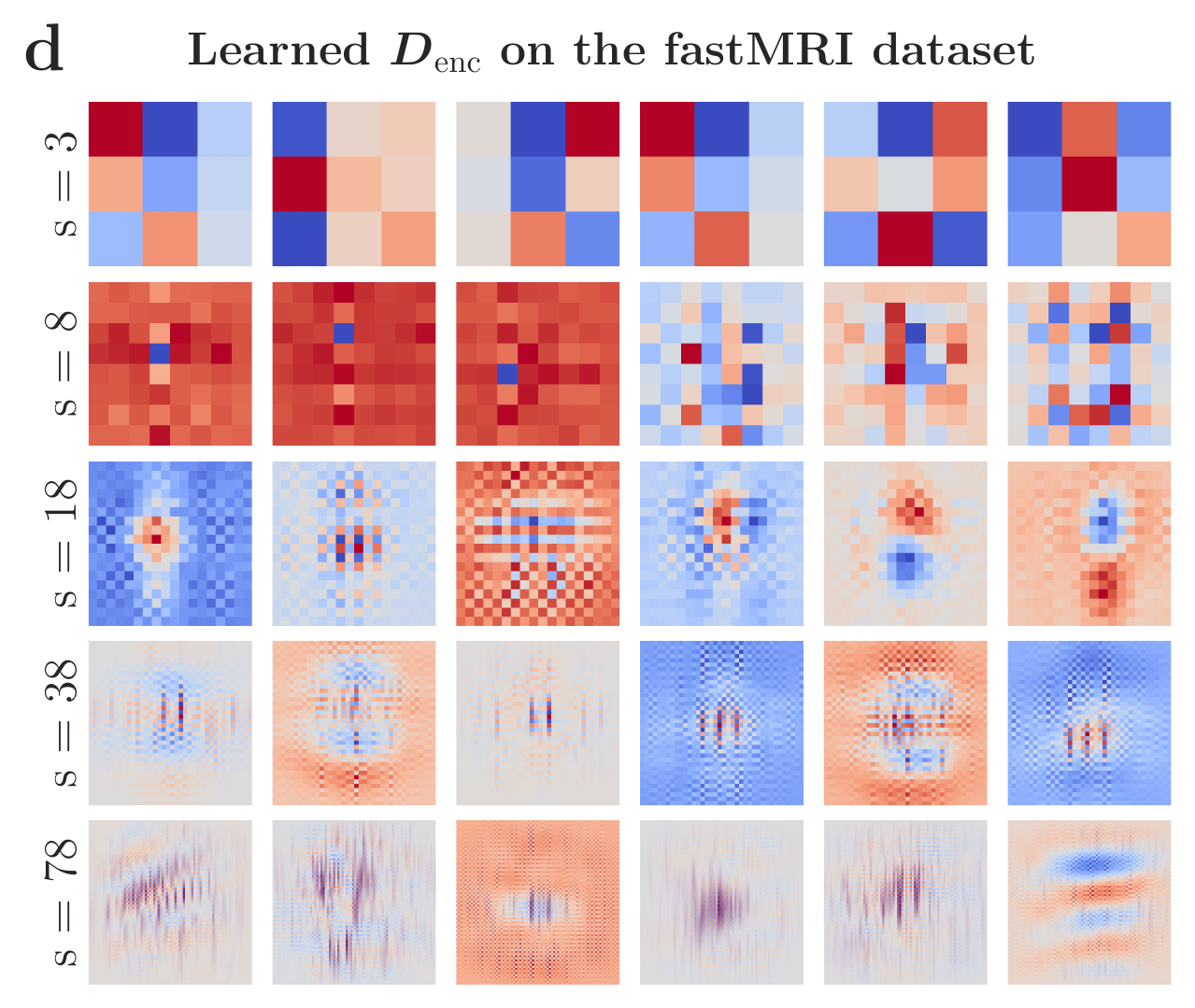} \\
\vspace*{0.5mm}
    \caption{\textbf{Atoms in a randomly initialized (panel~\textbf{a}) and learned $\encdict$ based on the derain dataset (panel~\textbf{b}), LoDoPab-CT (panel~\textbf{c}) and fastMRI (panel~\textbf{d}) dataset.} The visualization setup is identical to Figure~\ref{fig:fastmri_atoms} in the main text. \label{fig:encoder-atoms}}
\end{figure*}

\begin{figure*}[!ht] 
\centering
\includegraphics[width= 0.49 \textwidth]{figs/rand_dictionary.pdf} 
\includegraphics[width= 0.49 \textwidth]{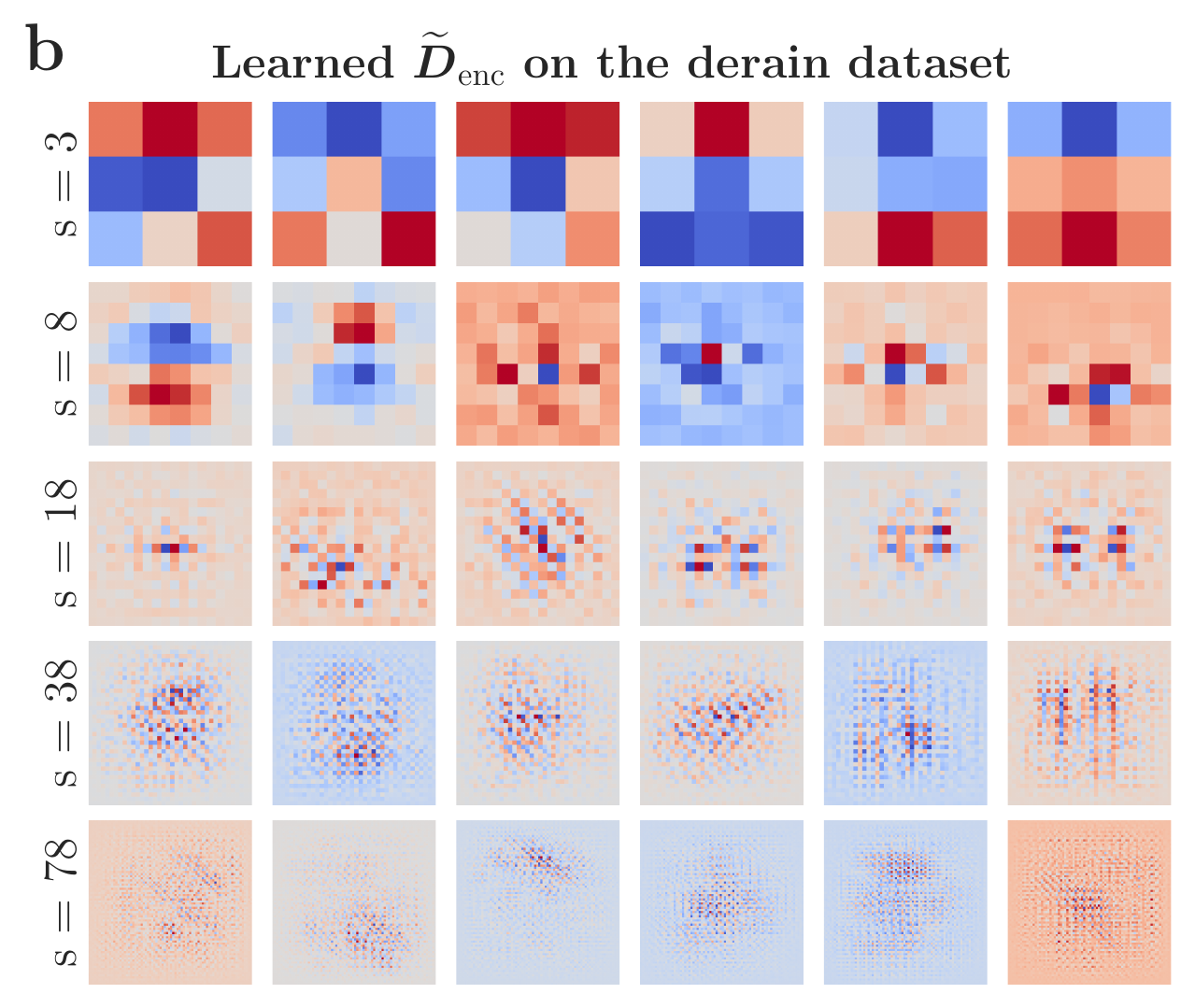} \\
\includegraphics[width= 0.49 \textwidth]{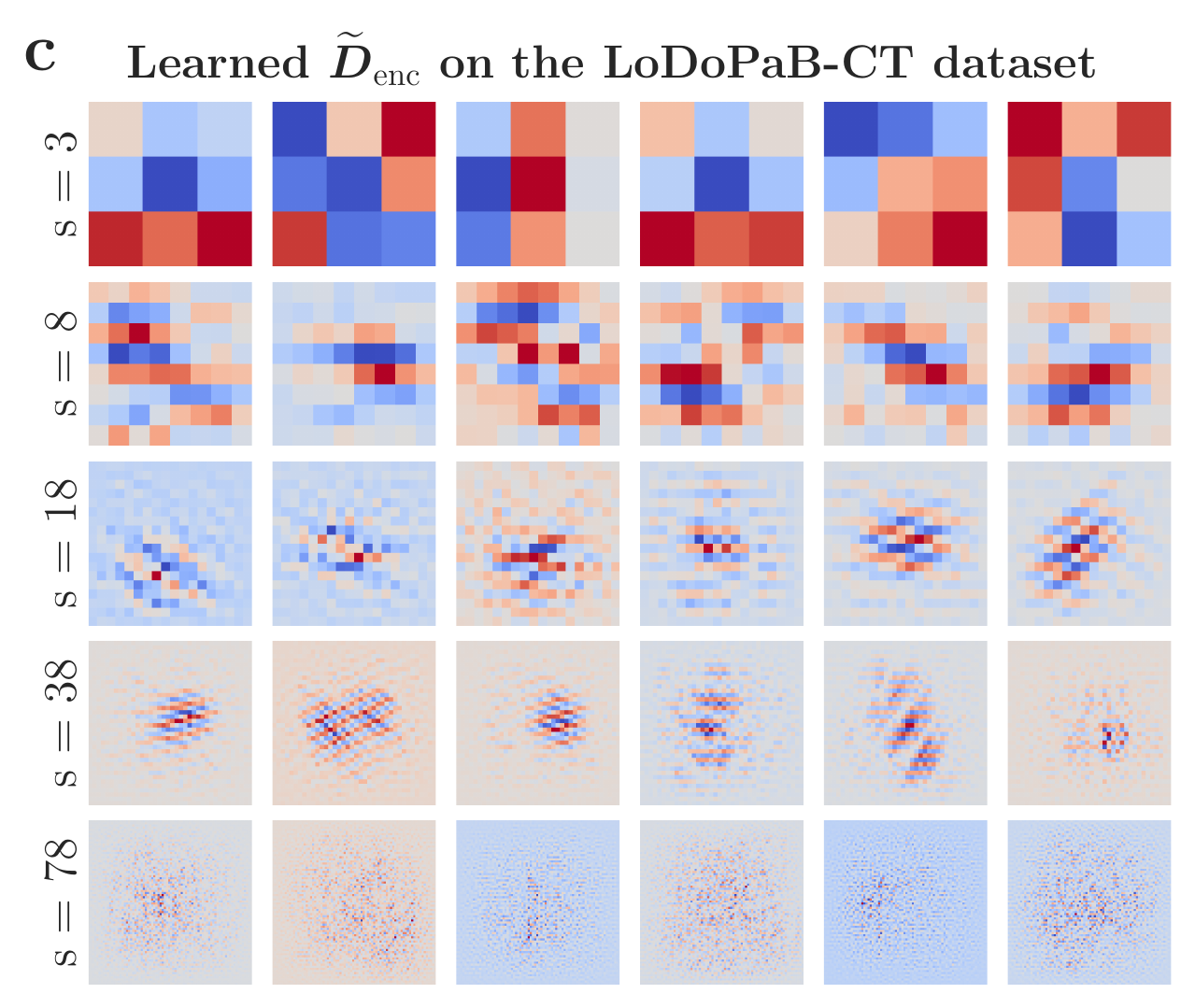}  
\includegraphics[width= 0.49 \textwidth]{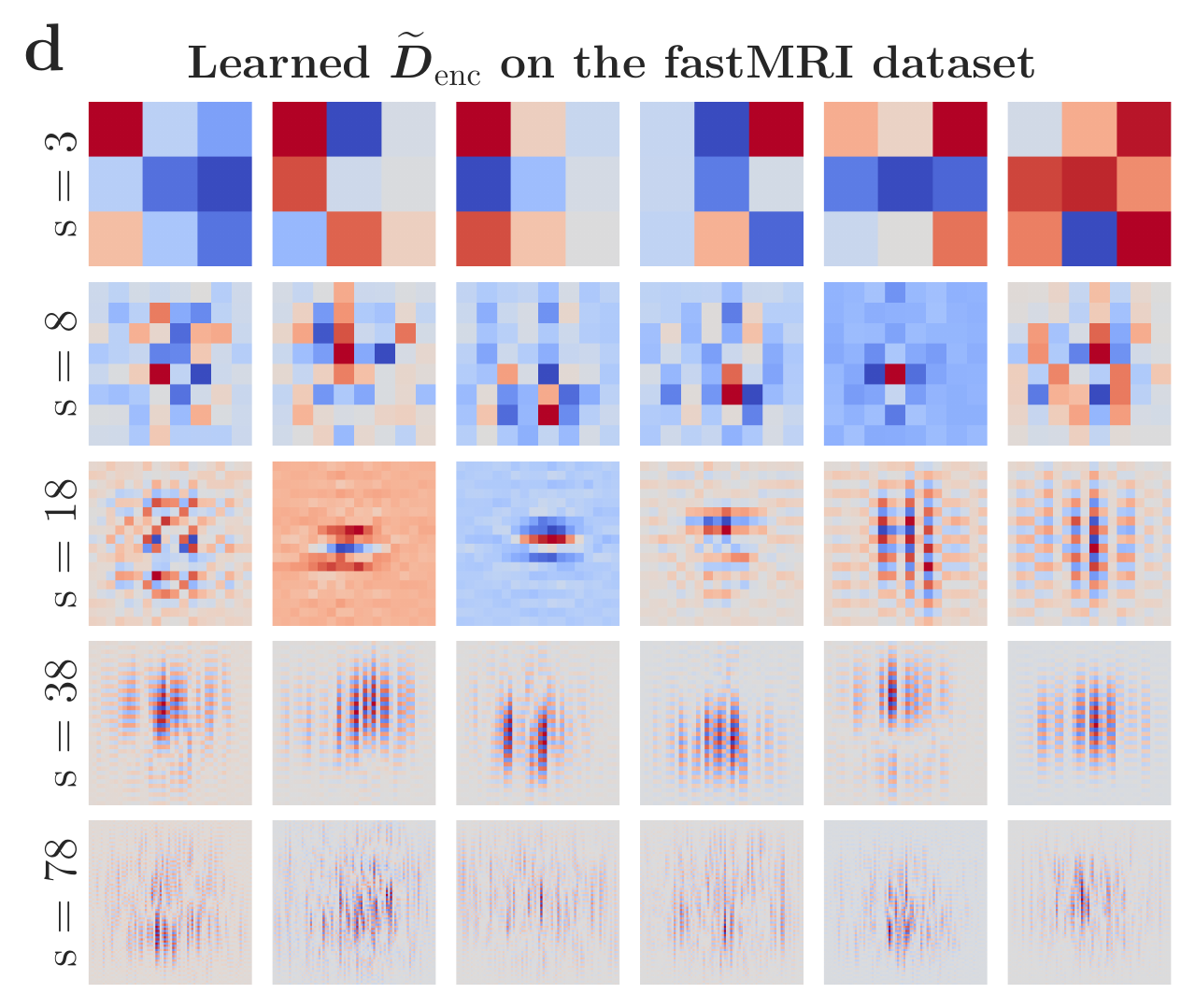} \\
\vspace*{0.5mm}
    \caption{\textbf{Atoms in a randomly initialized (panel~\textbf{a}) and learned $\adjencdict$ based on the derain dataset (panel~\textbf{b}), LoDoPab-CT (panel~\textbf{c}) and fastMRI (panel~\textbf{d}) dataset.} The visualization setup is identical to Figure~\ref{fig:fastmri_atoms} in the main text. \label{fig:adj-encoder-atoms}}
\end{figure*}

\section{Visualizing the representation of U-Nets \label{subsection:probe_unet}}

Similar to how we visualize dictionary atoms in Section \ref{subsection:probing}, we visualize prototypical images that U-Nets synthesize through its decoder branch $f^{\text{dec}}(\cdot, \vgamma)$. Concretely, we first prepare a set of indicator codes corresponding to different spatial resolutions as described in Section \ref{subsection:probing}. We then feed each indicator code $\vdelta$ into the decoder branch of a U-Net to yield $f^{\text{dec}}(\vdelta, \vgamma)$. Due to additive biases and batchnorm modules of the U-Net, the synthesized output $f^{\text{dec}}(\vdelta, \vgamma)$ has the same support as the full image. To focus on the region influenced by the indicator code, we thus display the support of $f^{\text{dec}}(\vdelta, \vgamma) - f^{\text{dec}}(\mathbf{0}, \vgamma)$, where $\mathbf{0}$ is an all-zero tensor; the purpose of this subtract is to offset those image values solely influenced by batchnorm and additive biases but not by the indicator code. These synthesized results are visualized in Figure \ref{fig:unet_atoms}. 

As it can be seen, compared to the randomly initialized U-Net (Figure \ref{fig:unet_atoms}\textbf{a}), the representations of learned U-Nets (Figure \ref{fig:unet_atoms}\textbf{b} and \textbf{c}) are organized in a more structured way at each scale. Compared to the (linear) representations learned by the MUSC, the (nonlinear) U-Net atoms much less resemble the classical oriented multiresolution systems such as curvelets.

\begin{figure*}[!ht] 
\centering
\includegraphics[width= 0.32 \textwidth]{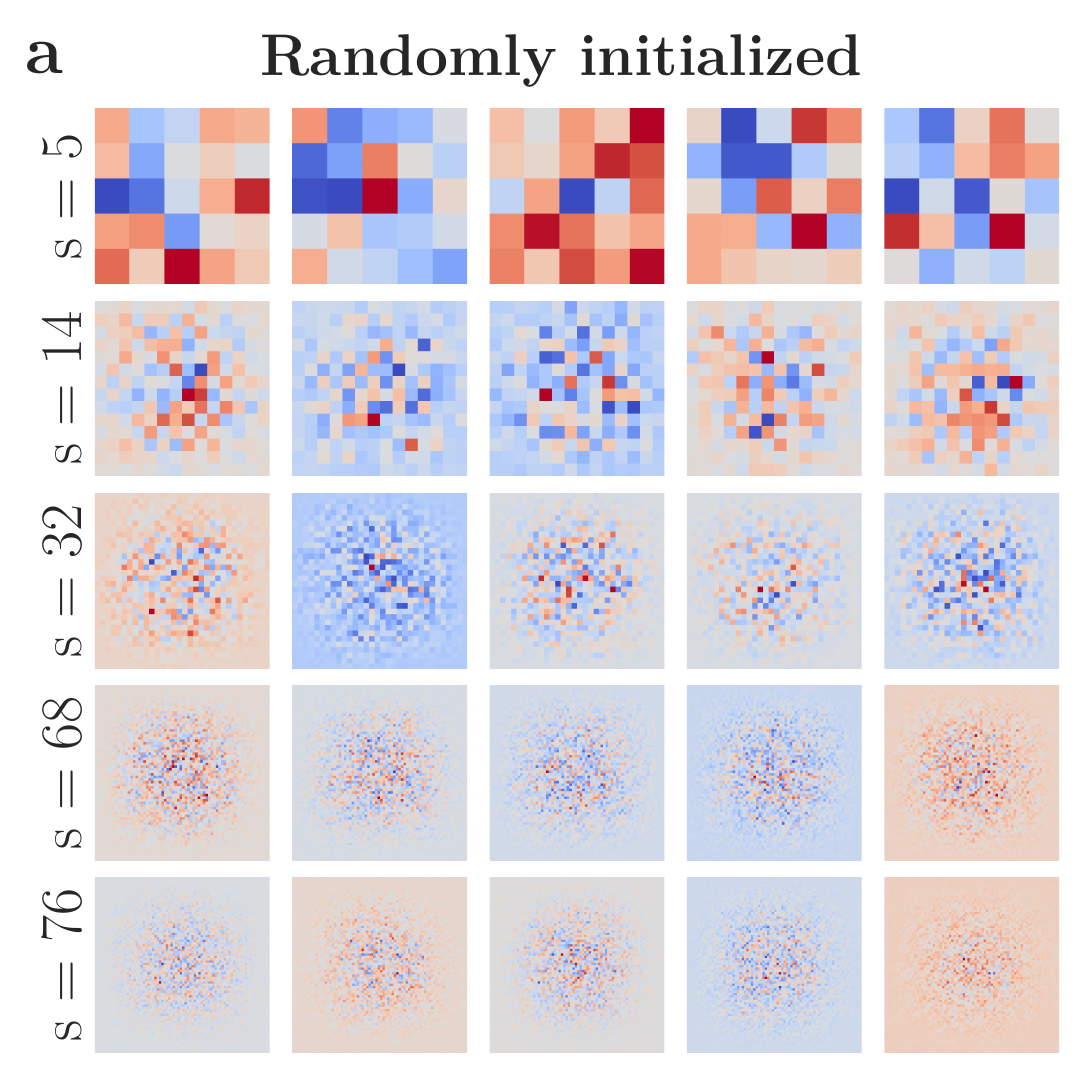}  
\includegraphics[width= 0.32 \textwidth]{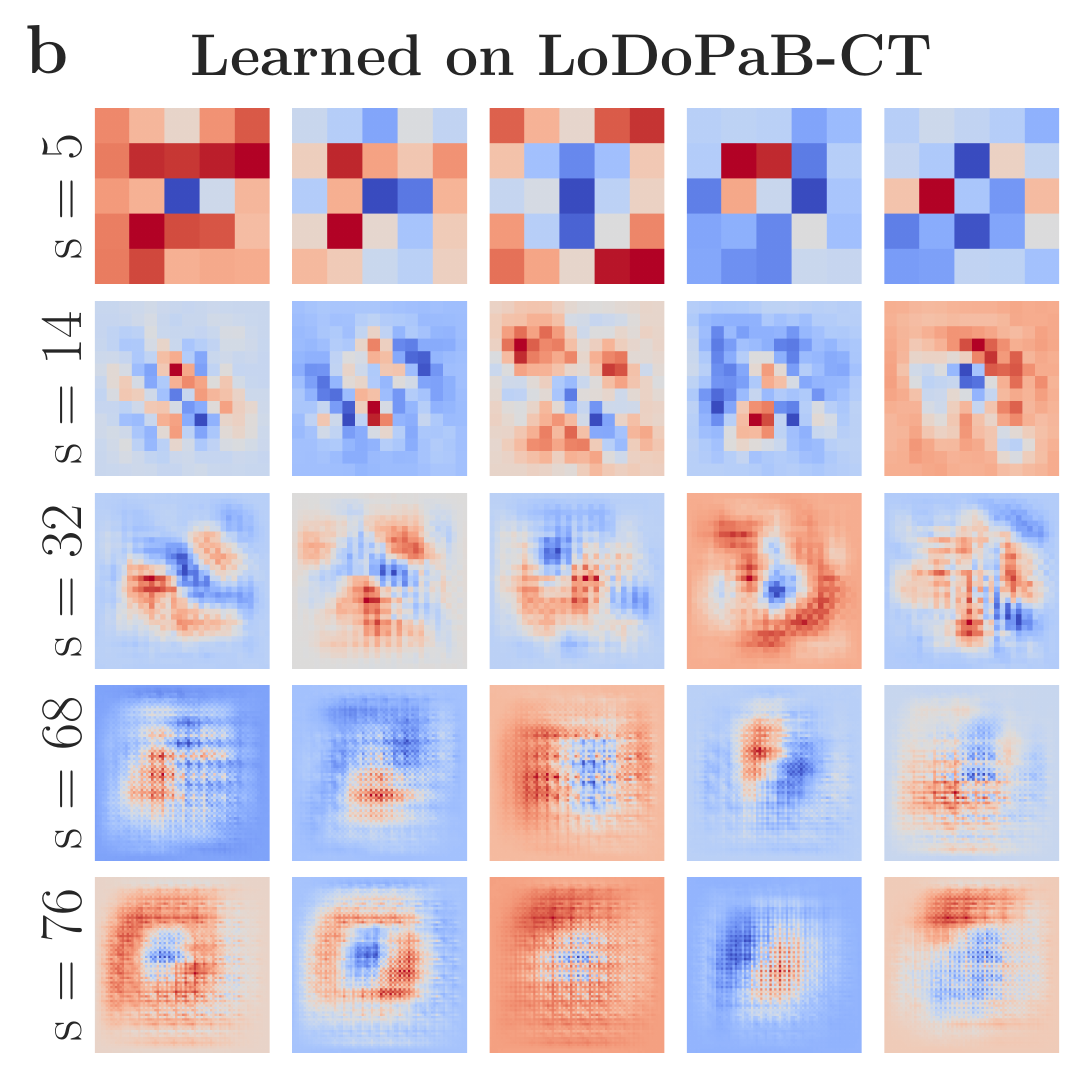} 
\includegraphics[width= 0.32 \textwidth]{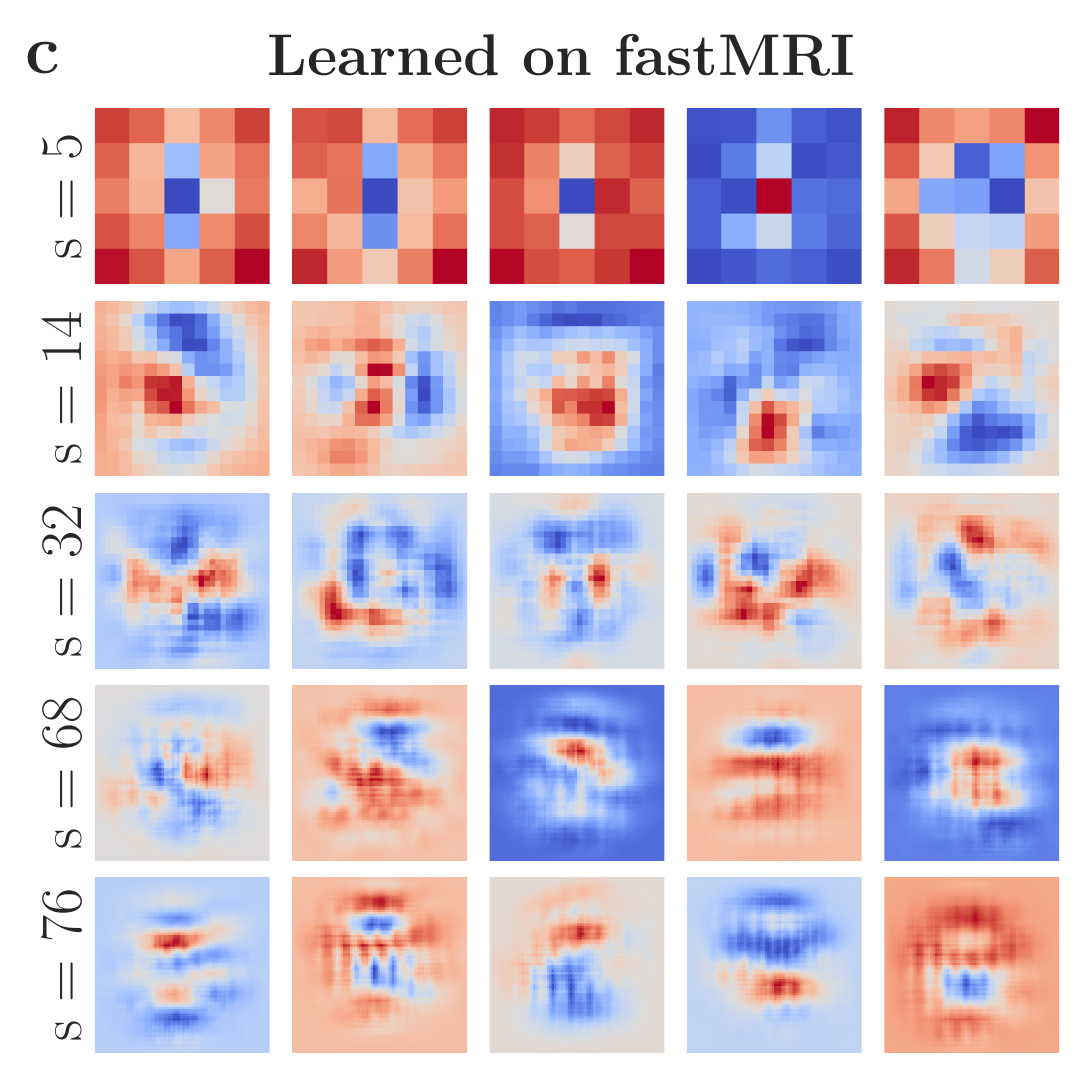} 
\vspace*{0.5mm}
    \caption{\textbf{Indicator-code-induced images of U-Nets with randomly initialized (left panel) and learned parameters based on the LoDoPab-CT (middle panel) and fastMRI (right panel).} For all panels, each row corresponds to a support size (denoted by $s$) of the receptive field of a convolution layer. For the visualization purpose, these images are normalized into the range $[-1,1]$. \label{fig:unet_atoms}}
\end{figure*}

\section{Super-resolution \label{appendix:superresolution}}

 CSC models achieve competitive performance in image super-resolution \cite{Gu2015convolutional, He2021image}. We train an out-of-the-box MUSC for this task to study the sparsity patterns of its learned representation in Section \ref{subsection:probing}. We follow the protocol in earlier work \cite{Schulter2015fast,Kim2016accurate,Lai2017deep} to train MUSC on the DIV2K dataset \cite{Agustsson2017ntire}. Low-resolution images are prepared by downscaling high-resolution images by a factor of four. We used bicubic interpolated images up-scaled from low-resolution images as model inputs and high-resolution images as model targets. The trained models were evaluated on standard datasets including Set-14 \cite{Zeyde2010single}, Set-5 \cite{Bowden2012low},  B-100 \cite{Martin2001database}, and Urban-100 \cite{Huang2015single}. Table~\ref{tab:super_resolution_results} shows the performance of the trained models. We observe that single-scale CSC has an edge of all other models, indicating the limited usefulness of large filters and the benefits of trading them off for a large number of small-support convolutional channels, as also discussed in the main text.

\begin{table}[!ht]
    \centering
    \begin{tabular}{l   c  c c c} \toprule
      &  Set-5 & Set-14 &  B-100 &  Urban-100 \\ \midrule
      ScSR \cite{Yang2010image}         &   {29.07}  &  {26.40}   &   {26.61}  &   {24.02}  \\
      A+ \cite{Timoftead2015justed}      &   {30.17} &  {26.94}   &   {26.81}  &   {24.29}  \\
     CSC \cite{Gu2015convolutional} (results taken from \cite{He2021image}) & 31.82 & 28.29 & 27.44 & 25.59 \\
     U-Net                     &  {30.52}   &  {27.43}   &  {27.04}   &  {24.68}   \\
     MUSC (ours)                        &  {31.30}   &  {28.01}   &  {27.20}   &  {25.08} \\ \midrule
    \end{tabular} %
	\vspace*{1mm}
    \caption{\textbf{Performance on the single-image super-resolution task with a scaling factor of four.}  \label{tab:super_resolution_results}}
\end{table}
\end{appendices}

\end{document}